\documentclass{article}

\usepackage[utf8]{inputenc}
\usepackage{amssymb, amsmath, amsthm}
\usepackage{tikz}
\usepackage{algorithm,algpseudocode}
\usepackage{hyperref}

\newtheorem{lemma}{Lemma}[section]
\newtheorem{remark}[lemma]{Remark}
\newtheorem{proposition}[lemma]{Proposition}

\newtheorem{corollary}[lemma]{Corollary}
\theoremstyle{remark}

\DeclareMathOperator*{\argmin}{argmin}
\DeclareMathOperator{\diag}{diag}

\newcommand{\real}{\mathbb{R}}
\newcommand{\float}{\mathbb{F}}

\newcommand{\dualp}[1]{\left\langle #1 \right\rangle} 

\newcommand{\domain}{\mathcal{D}}
\newcommand{\readOnly}{RO}
\newcommand{\inputClass}{\mathcal{T}_0}

\newcommand{\weight}{\omega}
\newcommand{\scale}{\gamma}
\newcommand{\thz}{t} 
\newcommand{\stopgrad}[1]{\backslash#1\backslash}

\newcommand{\tapeCurrent}{\mathcal{T}_\downarrow}
\newcommand{\tapeNext}{\mathcal{T}_\curvearrowright}

\newcommand{\headNext}{\mathcal{H}_\curvearrowright}
\newcommand{\headShift}{\mathcal{S}}

\begin{document}

\title{Universality of Gradient Descent Neural Network Training}

\author{G. Welper\footnote{Department of Mathematics, University of Central Florida, Orlando, FL 32816, USA, email \href{mailto:gerrit.welper@ucf.edu}{\texttt{gerrit.welper@ucf.edu}}. \newline 
This material is based upon work supported by the National Science Foundation under Grant No. 1912703.}}

\date{}
\maketitle

\begin{abstract}

  It has been observed that design choices of neural networks are often crucial for their successful optimization. In this article, we therefore discuss the question if it is always possible to redesign a neural network so that it trains well with gradient descent. This yields the following universality result: If, for a given network, there is any algorithm that can find good network weights for a classification task, then there exists an extension of this network that reproduces these weights and the corresponding forward output by mere gradient descent training. The construction is not intended for practical computations, but it provides some orientation on the possibilities of meta-learning and related approaches.

\end{abstract}



\section{Introduction}

Training neural networks with gradient descent is remarkably effective in a multitude of applications, see e.g. \cite{KrizhevskySutskeverHinton2012,HeZhangRenEtAl2016} for image classification, \cite{SilverHuangMaddisonEtAl2016} for reinforcement learning, \cite{SutskeverVinyalsLe2014,BahdanauChoBengio2015} for machine translation or \cite{GoodfellowBengioCourville2016} for a general overview. This is somewhat surprising because the objective function is generally non-convex and neural network training is $NP$-hard in the worst case \cite{BlumRivest1989}. 

The current literature contains a growing number of ideas and approaches to explain this phenomenon. Some experimental studies indicate that the loss function of common learning problems is more benign than one might assume on first sight \cite{GoodfellowVinyals2015,ZhangBengioHardtEtAl2017}. Other works \cite{ChoromanskaHenaffMathieuEtAl2015,LaurentBrecht2018,AroraCohenGolowichEtAl2019} provide a thorough understanding of simplified networks. For non-simplified general networks rigorous convergence results can be obtained under the assumption of over-parametrization \cite{SoudryCarmon2016,SafranShamir2018,LiLiang2018,Allen-ZhuLiSong2019,DuLeeLiEtAl2019}. Despite this progress, many practical networks do not satisfy all necessary assumptions and a solid understanding of the training behaviour remains a challenging question.

In this article, we address the problem from a different perspective. Instead of providing conditions on networks that guarantee convergence, we consider the question if we can modify a given network so that it trains well. This is loosely inspired by practical network training, where we do not consider one single fixed network either. More often, we experiment with a multitude of hyper-parameters and architectural elements such as drop-out, attention, skip connections, etc. until we obtain satisfactory results. The question is then not necessarily if every network trains well, but rather if we can find one that does.

To address this question, we provide the following universality result: For a given \emph{primary network} $f_\theta$ with weights $\theta$ and learning task, we assume that there is a Turing machine $TM$ that can compute good network weights $\bar{\theta} = TM(x,y)$ given inputs $x_i$ and labels $y_i$, $i=1, \dots, n$. This could be gradient descent training, global optimization methods or any other algorithm specifically tailored to the network and problem at hand. The existence of such a Turing machine merely asserts that some good training algorithms exits. We then construct an extended network that contains the primary network as a sub-network with the following three properties:
\begin{enumerate}
  \item Gradient descent training of the extended network adjusts the parameters $\theta$ of the primary network $f_\theta$ to the output $\bar{\theta} = TM(x,y)$ of the Turing machine.
  \item After gradient descent training, a forward execution of the extended network with input $x$ yields the output $f_{\bar{\theta}}(x)$, i.e. the output of the primary network $f_{\bar{\theta}}$ with weights $\bar{\theta} = TM(x,y)$ chosen by the Turing machine.
  \item The number of gradient descent steps matches the number of Turing machine steps.
\end{enumerate}
In summary, if there exists an algorithm that produces good weights from the learning data, these weights can be computed by gradient descent training on a properly extended network.

The extended network is carefully handcrafted, which, of course, is not intended as a practical algorithm, but rather an universality result to demonstrate the capabilities of gradient descent in combination with a properly chosen network. Form a practical perspective, there are multiple algorithms that aim at automatically generating good neural networks for a specific problem. Although these methods often do not optimize for gradient descent convergence directly, the extended network can be seen as a somewhat idealized potential outcome of such methods. One example is meta-learning \cite{HospedalesAntoniouMicaelliEtAl2020}, in particular methods such as MAML \cite{FinnAbbeelLevine2017} that pre-select network weights so that they can be adapted to a specific problem by one step of gradient descent. Another example is neural architecture search \cite{RealMooreSelleEtAl2017,ZophLe2017,ZophVasudevanShlensEtAl2018,WuDaiZhangEtAl2019}, which seeks to automatically generate neural networks with state of the art performance for a given learning task. In transfer learning \cite{DonahueJiaVinyalsEtAl2014,YosinskiCluneBengioEtAl2014} neural networks contain weights that have been pre-trained on related problems.

Another parallel can be drawn to biologically plausible learning methods \cite{LillicrapCowndenTweed2016,XiaoChenLiaoEtAl2019}. They often use dedicated ``feedback networks'' as a replacement for the back-propagation of gradients, as e.g. \cite{BellecScherrHajekEtAl2019}. Some examples include target propagation \cite{LeeZhangFischerEtAl2015} and synthetic gradients \cite{JaderbergCzarneckiOsinderoEtAl2017}. Although the extended networks of this article do use back-propagation,  a ``switch'' at the end of the network prevents gradients to be back-propagated directly to the primary network $f_\theta$. Instead, they are propagated into a parallel network where they trace the Turing machine computation and ultimately influence the weights of $f_\theta$. Therefore the network extension can also be understood as a ``feedback network''.

``Universality'' results for neural networks usually refer to their capacity to approximate arbitrary functions, with varying restrictions on network width, depth or activation functions. See \cite{Cybenko1989,HornikStinchcombeWhite1989,Barron1993} for some early results and e.g. \cite{Zhou2020} for convolutional networks, \cite{LuPuWangEtAl2017,HaninSellke2017} for deep networks and \cite{FinnLevine2017} for other learning regimes as model agnostic meta learning. More quantitative results are also available e.g. in \cite{DaubechiesDeVoreFoucartEtAl2019} and the references therein. In contrast to this body of work, the universality result in this article is concerned with a quite different question: We want to know if gradient descent training can always proceed to find desirable network weights, not if neural networks can approximate an arbitrary function.

There are also numerous connections between Neural networks and Turing Machines. Similar to the universal function approximation, neural networks can simulate arbitrary Turing Machines \cite{SiegelmannSontag1995}. Other results supplement neural networks with read and writable memory \cite{GravesWayneDanihelka2014,CollierBeel2018}, trainable by gradient descent. These latter neural Turing machines aim at making neural networks even more powerful in practical applications. Although they are a natural choice for the extended network of this article, for simplicity we confine ourselves to feed forward networks with sufficient width to hold the full tape during the computation of $TM(x,y)$ for all inputs $x$ and labels $y$ of fixed size.

The paper is organized as follows: Section \ref{sec:main} contains the main result of the paper and Section \ref{sec:extended-network-overview} contains a brief overview over the construction of the extended network. In Section \ref{sec:TM-via-gradient-descent}, we construct two loss functions that allow us to trace steps of Turing machines by gradient descent. Finally in Section \ref{sec:supervised-learning}, we prove the main result of the article.

\section{The Main Result}
\label{sec:main}

Let us consider the following standard learning task: Given samples $x_i \in \float^M$, $i=1, \dots, n$ and corresponding labels $y_i \in \float^m$, we want to train the parameters $\theta$ of a \emph{primary network} $f_\theta: \float^M \to \float^m$ so that $f_\theta(x)$ predicts labels for any input $x \in \float^M$. In order to avoid problems with Turing computability, we use some finite  precision floating point numbers $\float$ instead of real numbers $\real$. The structure of the parametric function $f_\theta$ is not important for the following considerations, although we are mainly interested in neural networks with input $x \in \float^M$ and network weights $\theta \in \float^W$. For notational convenience, we combine all $x_i$ and $y_i$ into two matrices $x \in \float^{n \times M}$ and $y \in \float^{n \times m}$.

We assume that there is a computable function $TM(x,y) \to \theta$ that for given inputs $x$ and labels $y$ of a learning task produces suitable weights $\theta$. For example, the function $TM$ could simply return the result of a standard gradient descent training, a global optimization method or any algorithm that is specialized for the problem at hand. Anyways, we do not consider the meaning of ``suitable'' or the choice of the algorithm any further. Instead, we show that whatever algorithm is chosen, its resulting weights $\theta$ can be reproduced by gradient descent training of an extended network. 

This \emph{extended network} is a parametric function $F$ with the following properties:
\begin{equation}
  \begin{aligned}
    & \text{$F: \real^M \times \domain \to \real^m$ for some parameter domain $\domain$}
    \\
    & \text{$F$ is composed of: $f_\theta$, $ReLU$, square root and product non-linearities,}\\
    & \quad\quad \text{stop gradient operations and quantization/de-quantization.}
  \end{aligned}
  \label{eq:intro:enlarged-network-properties}
\end{equation}
The stop gradient operations are used to provide a directionality for read/write operations of the Turing machine and are readily available in contemporary deep learning libraries. The quantization is never differentiated and used to transfer floating point inputs $\float$ to bit sequences for the Turing machine. Of course on any computer, internally any number is already in binary format, which can be passed directly to the Turing machine. Formally, one can also use a neural network to transfer floating point numbers to a bit sequence and vice versa as discussed in Appendix \ref{sec:quantization}.

The weights $(s, \thz) \in \domain$ of the extended network are constrained to the Cartesian product $\domain = S \times \real^{4\tau + n \times m}$ of a simplex $S$ and a vector space. The dimensions and contents will be determined later, at this point we merely need the simplex structure to define an appropriate gradient descent method. To this end, we use the least squares loss
\begin{equation}
  \ell(s, \thz) = \frac{1}{2} \|F(x,(s, \thz)) - y\|^2
  \label{eq:lsq-loss}
\end{equation}
on the original dataset $(x,y)$. Since the weight $s$ is constrained to a simplex, we use the conditional gradient, or Frank-Wolfe, algorithm 
\begin{equation}
  \begin{aligned}
    s_{k+1} & = s_k - [\argmin_{\sigma \in S} \partial_{\sigma - s_k} \ell(s_k, \thz_k) - s_k] \\
    \thz_{k+1} & = \thz_k - \alpha \odot \nabla_t \ell(s_k, \thz_k)\\
  \end{aligned}
  \label{eq:intro:gd-tm-supervised-learning}
\end{equation}
with some fixed learning rate $\alpha \in \real^{|\domain|}$ and component wise multiplication $\odot$.

In addition to the extended network, the construction of this article provides some explicit stopping criterion
\begin{equation}
  \ell(s_k, \thz_k) \le B_{stop}
  \label{eq:intro:gd-tm-supervised-stopping}
\end{equation}
for some $B_{stop}$ bigger than the final training error $\|f_{TM(x,y)}(x) - y)\| \le B_{stop}$. The following proposition is the main result of this article: We train the extended network with the given loss, gradient descent method, learning rate and stopping criterion. Afterwards, computing a forward pass of the extended network yields the same result as if we would run the primary network $f_\theta$ with weights $\theta = TM(x,y)$ chosen by the Turing machine.

\begin{proposition}
  \label{prop:learn-tm}
  For any number of samples $n$ and finite precision floating point numbers $\float$, let $TM: \float^{n \times M} \times \float^{n\times m} \to \float^p$ be a computable function given by a Turing machine that halts for every input, using less than $\tau = \tau(n)$ tape symbols at any step of the computation. Assume that the labels satisfy 
  \begin{equation}
    \|y\|_2^2 \ge \epsilon
    \label{eq:labels-size}
  \end{equation} 
  for some $\epsilon > 0$. Then, there is an extended network $F$ that satisfies all properties in \eqref{eq:intro:enlarged-network-properties}, with initial values $(s_0, \thz_0)$, fixed learning rate $\alpha$ and stopping threshold $B_{stop}$, depending only on $TM$ such that after gradient descent training \eqref{eq:intro:gd-tm-supervised-learning} with stopping criterion \eqref{eq:intro:gd-tm-supervised-stopping} we have
  \[
    F(x,(s_K, \thz_K)) = f_{TM(x,y)}(x)
  \]
  for the final parameter $(s_K, \thz_K) \in \domain$ of the gradient descent method. The gradient descent stopping criterion is met in $k_t+1$ steps, where $k_t$ is the number of steps the Turing machine's uses to compute $TM(x,y)$.

\end{proposition}

An overview over the construction is provided in Section \ref{sec:extended-network-overview}. The proof of the proposition, as well as the following two corollaries is deferred to Sections \ref{sec:TM-via-gradient-descent} and \ref{sec:supervised-learning}. Bounds on the size of the extended network depend on its construction. First, if we insist that the gradient descent loss is strictly decreasing, we obtain the following.

\begin{corollary}
  \label{cor:learn-tm-tape-internal}
  Assume the Turing machine $TM$ has two tape symbols, $d \ge 2$ tapes and $|Q|$ states in its finite control. Then, the extended network in Proposition \ref{prop:learn-tm} can be chosen with non-increasing gradient descent loss $\ell(s_k, \thz_k)$ and less than $12$ extra layers of width smaller than $|Q| 2^{d\tau} \tau^d + \mathcal{O}(nm)$ in addition to $f_\theta$ .
\end{corollary}

This result is included in the paper because it contains the basic ideas to trace a Turing machine by gradient descent without adding too much technicalities for read/write operations to the tapes. However, the result itself is questionable because the extra layers are large enough to encode all possible states of the Turing machine for the given input sizes. This includes states with all possible inputs and outputs written in the tapes and therefore one could just as well store all possible input/output relations (these are finite because we work with floating point numbers $\float$), which is obviously infeasible.

The following corollary provides a similar result with a much smaller extended network, which scales linearly in the required number of tape symbols $\tau$ and therefore has a comparable size to the Turing machine itself. Unlike the last corollary, the loss function is no longer monotonically decreasing with regard to the gradient descent steps and does not work with a line search.

\begin{corollary}
  \label{cor:learn-tm-tape-external}
  Assume the Turing machine $TM$ has two tape symbols, $d \ge 2$ tapes and $|Q|$ states in its finite control. If the gradient descent loss $\ell(s_k, \thz_k)$ is allowed to be non-monotonic, the size of the extended network $F$ in Proposition \ref{prop:learn-tm} can be constrained to less than $12$ extra layers of width smaller than $2^d |Q| + 2 d \tau + \mathcal{O}(nm)$,  in addition to $f_\theta$ .
\end{corollary}

This construction contains two network weights $T$ and $H$ as components of $\thz$ for the Turing machine's tape and head position. In principle, we can allow these to be of infinite dimension leading to infinite tape length $\tau = \infty$ as for regular Turing machines.

\begin{remark}
  The extra width $\mathcal{O}(nm)$ is required to read the full set of labels into the Turing machine $TM$. Many practical algorithms work with batches of labels, which could in principle also be implemented with a similar extended network. In this case, we would only need an extra width of $\mathcal{O}( [\text{batch size}] \cdot m)$.
\end{remark}

\section{Extended Network: Overview}
\label{sec:extended-network-overview}

This section provides a brief overview over the construction of the extended network. A detailed construction and the proof of the main results is given in Sections \ref{sec:TM-via-gradient-descent} and \ref{sec:supervised-learning} below.

The construction proceeds in two steps: First, we construct a loss function $\ell_{TM}$ that allows us to trace the states of a generic Turing machine by gradient descent steps. Unlike typical loss functions, it does not compare predictions to labeled training data. Instead, it is build from standard neural network components and therefore, in the second step, used as a sub-network together with the primary network and some ``switches'' to build the extended network.

The construction of $\ell_{TM}$ is given in Section \ref{sec:TM-via-gradient-descent}. To summarize, we associate each state of the Turing machine with a vertex of a simplex. Depending on the construction in Corollaries \ref{cor:learn-tm-tape-internal} or \ref{cor:learn-tm-tape-external} the state includes the tape or only the finite control and the tape symbols at the head position with secondary tape and head variables. The loss function maps the simplex to the real numbers and is chosen so that gradient descent steps with unit learning rate remain on vertices if started at an initial vertex. Therefore, the gradient descent steps can be associated with states of the Turing machine, which we referred to as ``tracing'' the Turing machine, above.

In order to define $\ell_{TM}$, we subdivide the simplex into corner simplices, one for each vertex, and the remaining interior part. On each corner we can freely assign a piecewise linear loss function and then extend it continuously to the full simplex. Since the gradient descent steps remain on the vertices, this allows us to control the gradients and carve ``barriers'' and ``slides'' into $\ell_{TM}$ that guides the gradient descent updates along the states of the Turing machine execution. If the states do not contain the full tape, some extra least squares terms are added for reading and writing to a secondary tape variable.

The extended network is composed of the primary network $f_\theta$ and the Turing machine network $\ell_{TM}$ as shown in Figure \ref{fig:nn-extended}. For the time being, we only consider a rudimentary overview and defer a more detailed description to Section \ref{sec:extended-network-construction} below.

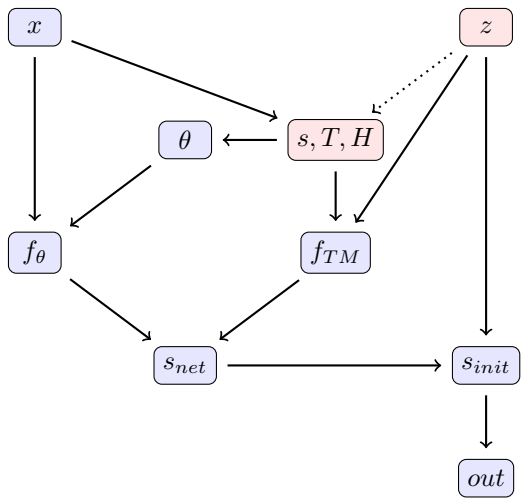
\begin{figure}
  \begin{center}
    \begin{tikzpicture}[
      in_out/.style={rectangle,draw,rounded corners=3pt,minimum width=0.5cm,minimum height=0.5cm,fill=red!10},
      layer/.style={rectangle,draw,rounded corners=3pt,minimum width=0.7cm,minimum height=0.5cm,fill=blue!10},
      arrow/.style={thick,->,shorten >= 4pt,shorten <= 4pt},
      dotarrow/.style={thick,dotted,->,shorten >= 4pt,shorten <= 4pt},
    ]
  
      \newcommand{\dx}{2}
      \newcommand{\dy}{1.5}
  
      \node[layer]  (x) at (0*\dx,4*\dy) {$x$};
      \node[layer,fill=red!10]  (z) at (3*\dx,4*\dy) {$z$};

      \node[layer]  (nnWeights) at (1*\dx,3*\dy) {$\theta$};
      \node[layer,fill=red!10]  (s) at (2*\dx,3*\dy) {$s,T,H$};
  
      \node[layer]  (nn) at (0*\dx,2*\dy) {$f_\theta$};
      \node[layer]  (tm) at (2*\dx,2*\dy) {$f_{TM}$};
  
      \node[layer]  (switchNetwork) at (1*\dx,1*\dy) {$s_{net}$};
      \node[layer]  (switchInit) at (3*\dx,1*\dy) {$s_{init}$};
  
      \node[layer]  (out) at (3*\dx,0*\dy) {$out$};
  
      \draw[arrow] (x) -- (nn);
      \draw[arrow] (x) -- (s);
      \draw[dotarrow] (z) -- (s);
      \draw[arrow] (z) -- (tm);
      \draw[arrow] (z) -- (switchInit);

      \draw[arrow] (s) -- (tm);
      \draw[arrow] (s) -- (nnWeights);
      \draw[arrow] (nnWeights) -- (nn);
  
      \draw[arrow] (nn) -- (switchNetwork);
      \draw[arrow] (tm) -- (switchNetwork);
  
      \draw[arrow] (switchNetwork) -- (switchInit);
  
      \draw[arrow] (switchInit) -- (out);
      
    \end{tikzpicture}
  \end{center}
  \caption{Extended neural network for gradient descent training}
  \label{fig:nn-extended}
\end{figure}

The bottom two layers $s_{net}$ and $s_{init}$ contain two switches, which select the branches of the network that are passed to the output. The first step of gradient descent training writes a copy of the labels $y$ into the weights/node $z$. This copy is used as input for the Turing machine $TM$ because the actual labels $y$ are only implicitly available through the loss function.

The first training step also flips the switches so that only the $f_{TM}$ branch is passed to the output and the remaining gradient descent steps trace the Turing machine as described above. $f_{TM}$ is a wrapper around $\ell_{TM}$ that matches the output dimensions from the scalar loss to the dimension of the labels $y$.

Finally, once the Turing machine reaches a halting state, the last gradient descent step flips the switches again so that only the primary network is passed to the output. The parameters $\theta$ are no longer network weights but read from the weights $s,T,H$ representing the Turing machine's state and tape, containing $TM(x,y)$ after halting. Therefore, any further forward pass through the network computes the result $f_{TM(x,y)}$ of the primary network with weights computed by the Turing machine.

Note that during the entire training all gradients are passed only through the subnetwork $f_{TM}$, but never through the primary network $f_\theta$. This is reminiscent of approaches in biologically plausible learning, which use separate ``feedback networks'' for the adjustment of network weights. Unlike other approaches in the literature, this effect is solely generated by the network architecture, not by a modification of the back-propagation algorithm.

\section{Simulate Turing Machines via Gradient Descent}
\label{sec:TM-via-gradient-descent}

\subsection{Turing Machines}
\label{sec:turing-machines}

Let us first fix the Turing machine TM. Without loss of generality, we assume that it has $d$ tapes and only two tape symbols $\{-1,1\}$, including the blank.
\begin{enumerate}
  \item $Q$: finite and non-empty set of states.
  \item $q_0 \in Q$: Initial state.
  \item $F \subset Q$: accepting states.
  \item Tape symbols: $\Gamma = \{-1,1\}$.
  \item $\delta: Q \times \Gamma^d \to Q \times \Gamma^d \times \{-1,1\}^d$: transition function with $-1$ and $1$ denoting left and right shift of the head.
\end{enumerate}
We denote the component functions of $\delta$ by $\delta_i(q,t)$ so that $\delta(q,t) = [\delta_1(q,t), \delta_2(q,t), \delta_3(q,t)]$. 

We make two more assumptions on the Turing machine. First a subset $\readOnly \subset \{1, \dots, d\}$ of the tapes is read-only. This will usually be the tape holding the inputs, whereas the outputs are placed on the remaining tapes. Second, for a class $\inputClass$ of well formed inputs of interest, we assume that the Turing machine halts in finite time and uses at most $\tau$ consecutive tape entries on each tape at any time during the computation. As for any real computer, we then assume that the tapes have finite length $\tau$. 

Typically, the class $\inputClass$ consists of encodings of the training data $x$ and $y$ and the tape size $\tau$ depends on the data size. For the construction in Corollary \ref{cor:learn-tm-tape-external}, it is allowed to be infinite.

\subsection{Technical preliminaries}

The loss function $\ell_{TM}$ to trace the Turing machine $TM$ with gradient descent is constructed with the help of piecewise linear functions on simplices. This sections contains some explicit formulas for their evaluation, derivatives and optimal descent directions. To this end, let $S$ be a simplex with vertices $e_v$, $v \in V$ for some index set $V$ and assume that $\ell: S \to \real$ is affine. The vertices $e_v$ will be standard basis vectors below, but this is not necessary for this section. Then, for $x \in S$ with barycentric coordinates
\[
  \begin{aligned}
    x & = \sum_{v \in V} x_v e_v, &
    \sum_{v \in V} x_v & = 1, &
    x_v & \ge 0
  \end{aligned}
\]
we can evaluate $\ell$ by 
\[
  \ell(x) = \sum_{v \in V} x_v \ell(e_v).
\]
Moreover, affinity implies that for any convex combination we have 
\begin{equation*}
  \ell \big((1-\lambda)x + \lambda y \big) 
  = (1-\lambda) \ell(x) + \lambda \ell(y).
\end{equation*}
Therefore, for any two points $x,y$ in the simplex $S$, the one-sided directional derivative in the direction $y-x$ is given by
\begin{equation}
  \begin{aligned}
    \partial_{y-x} \ell(x) & = \lim_{\substack{h \to 0 \\ h \ge 0}} \frac{1}{h} \left[ \ell(x + h(y-x)) - \ell(x) \right]
    = \lim_{\substack{h \to 0 \\ h \ge 0}} \frac{1}{h} \left[ \ell( (1-h)x + hy) - \ell(x) \right]
    \\
    & = \lim_{\substack{h \to 0 \\ h \ge 0}} \frac{1}{h} \left[ (1-h)\ell(x) + h\ell(y)) - \ell(x) \right]
    \\
    & = \ell(y) - \ell(x).
  \end{aligned}
  \label{eq:directional-derivative}
\end{equation}
For gradient descent on simplices, we optimize over directions that point into the simplex. However, the loss function will not be affine on the entire simplex, but only in its corners. To this end, for $0 < \mu \le 1$, the \emph{corner} at $v$ is a sub-simplex defined by 
\begin{align*}
  S_v^\mu & := \operatorname{conv} \Big( \{e_v\} \cup \{e_{vw}^\mu | \, v \ne w \in V \} \Big), &
  e_{vw}^\mu := (1-\mu) e_v + \mu e_w,
\end{align*}
with ``$\operatorname{conv}$'' denoting the convex hull. In the following, we denote the vertices of a simplex $S$ by $V(S)$. If $\mu=1/2$, we also use the abbreviations
\begin{align*}
  S_v & := S_v^{1/2}, & e_{vw} & := e_{vw}^{1/2}.
\end{align*}
The next lemma provides explicit formulas for optimal descent directions with eventual extra quadratic terms.

\begin{lemma} 
  \label{lemma:simplex-derivative}
  Let $e_v$ be a vertex of $S \subset \real^N$ and $\ell$ be a loss function that is affine on the corner $S_v^\mu$ for some $0 < \mu \le 1$. Then for $A \in \real^{M \times N}$ and $y \in \real^M$, we have
  \begin{equation*}
    \argmin_{z \in S} \left\{ \partial_{z-e_v} \left[ \ell(x) + \frac{1}{2} \|Ax-y\|^2 \right]_{x = e_v} \right\} - e_v = \frac{1}{\mu} (e_w^\mu-e_v) = e_w - e_v,
  \end{equation*}
  with 
  \begin{align}
    e_w^\mu  & = \argmin_{b_u \in V(S_v^\mu)} \left[ \ell(b_u) + (Ae_v-y)^T A b_u \right],
    \label{eq:lemma:simplex-derivative-1}
    \\
    e_w & = \argmin_{e_u \in V(S)} \left[ \ell( [1-\mu] e_v + \mu e_u) + \mu (Ae_v-y)^T A e_u \right].
    \label{eq:lemma:simplex-derivative-2}
  \end{align}
\end{lemma}

\begin{proof}

Let us abbreviate $f(x) := \ell(x) + \frac{1}{2} \|Ax - y\|^2$. We first rescale the directional derivatives to the corner $S_v^\mu$. To this end, note that
\[
  \partial_{z - e_v} f(e_v) 
  = \frac{1}{\mu} \partial_{\mu(z - e_v)} f(e_v) 
  = \frac{1}{\mu} \partial_{[(1-\mu)e_v + \mu z] - e_v)} f(e_v),
\]
where the point $a := (1-\mu) e_v + \mu z$ is in the corner $S_v^\mu$ if and only if $z$ is in $S$. Hence, the last equation and the identity $\frac{1}{\mu} (a - e_v) = z - e_v$ imply that
\begin{equation}
  \argmin_{z \in S} \left\{ \partial_{z - e_v} f(e_v) \right\} - e_v 
  = \frac{1}{\mu} \left( \argmin_{a \in S_v^\mu} \left\{ \partial_{a - e_v} f(e_v) \right\} - e_v \right).
  \label{eq:lemma:proof:simplex-derivative-1}
\end{equation}
Since $\ell$ is affine in the corner $S_v^\mu$, we use \eqref{eq:directional-derivative} to simplify the directional derivative to
\[
  \partial_{a - e_v} \ell(e_v)
  = \ell(a) - \ell(e_v)
  = \sum_{u \in V} a_u \ell(b_u) - \ell(v)
\]
for any $a \in S_v^\mu$ with barycentric coordinates $a_u$ with respect to the vertices $\{b_u | \, u \in V\} = V(S_v^\mu)$ of $S_v^\mu$. Likewise, we have
\begin{multline*}
  \partial_{a - e_v} \left[ \frac{1}{2} \|Ae_v - y\|^2 \right] 
  = (Ae_v - y)^T A (a-e_v) 
  \\
  = \left( \sum_{u \in V} a_u (Ae_v - y)^T b_u \right) - (Ae_v - y)^T Ae_v.
\end{multline*}
In the last two equations the respective last terms $\dots - \ell(e_v)$ and $\dots - (Ae_v-y)^TAe_v$ do not depend on $a$ and therefore the minimizer is given by
\[
  \argmin_{a \in S_v^\mu} \partial_{a-e_v} f(e_v) = \argmin_{a \in S_v^\mu} \sum_{u \in V} a_u \left[ \ell(b_u) + (Ae_v-y)^T A b_u \right].
\]
The right hand side is minimal, if $a$ puts all its weight on the smallest summand, so that
\begin{align*}
  \argmin_{a \in S_v^\mu} \partial_{a-e_v} f(e_v) & = b_u, & u & = \argmin_{b_u \in V(S_v^\mu)} \left[ \ell(b_u) + (Ae_v-y)^T A b_u \right].
\end{align*}
Since $b_u$ are the vertices of $S_v^\mu$, together with \eqref{eq:lemma:proof:simplex-derivative-1} this directly shows \eqref{eq:lemma:simplex-derivative-1}. In order to show \eqref{eq:lemma:simplex-derivative-2}, note that every vertex $b_u \in V(S_v^\mu)$ is a convex combination $(1-\mu) e_v + \mu e_u$ of vertices of $S$. Therefore, we have
\[
  (Ae_v - y)^T A b_u
  = (1-\mu)(Ae_v - y)^T A e_v + \mu (Ae_v - y)^T A e_u.
\]
Using that the first summand of the right hand side is independent of $u$ shows \eqref{eq:lemma:simplex-derivative-2}.

\end{proof}

\subsection{Tracing Turing Machines: No Extra Tape Variable}
\label{sec:tm-descent-no-tape-variables}

In this section we construct a loss function that is used in the construction of Corollary \ref{cor:learn-tm-tape-internal} and allows a gradient descent method to trace the steps of the Turing machine $TM$.  The method is simple but also very costly, therefore we consider a related but more efficient approach in Section \ref{sec:tm-descent-with-tape-variables} below.

\paragraph{The Loss Function}

By the assumption in Section \ref{sec:turing-machines}, for all relevant inputs in $\inputClass$ the Turing machine halts in finite time using at most $\tau$ tape entries at any time during the computation. Hence, we can represent every relevant computation by a finite directed graph. Its vertices $V$ are all states of the Turing machine with the given size constraint and the edges $E$ encode the computational steps, i.e. there is an edge from vertices $v$ to $w$ if and only if the state $w$ follows after $v$ in the Turing machine execution.

In order to construct the corresponding loss function, we first assign weights $\weight_v$ to vertices and $ \weight_{vw}$ pairs of vertices, including non-edges, so that their minimization mimics the Turing machine execution. To this end, let 
\begin{equation}
  \begin{aligned}
    \weight_v & > \weight_{vw}, & & \text{for }(v,w) \in E \\
    \weight_{uv} & > \weight_{vw}, & & \text{for } (u,v), (v,w) \in E \\
    \weight_{vw} & = B, & & \text{for }(v,w) \not\in E\text{ and }(w,v) \not\in E \\
    \weight_{vw} & = \weight_{wv}, & & v,\, w \in V,
  \end{aligned}
  \label{eq:graph-edge-weights}
\end{equation}
where $B$ is an upper bound
\[
  B > \max \left\{\max_{v \in V} \weight_v, \max_{(v,w) \in E} \weight_{vw} \right\}.
\]
of all vertices and computationally legitimate edge weights. Although for the time being the problem is still discrete, we consider moving between vertices along the pairs with minimal weights $\weight_{vw}$, or staying in place if the current vertex weight $\weight_v$ is smaller than the outgoing connections. The first inequality of \eqref{eq:graph-edge-weights} ensures that the correct computational step is preferred over staying in place, the second inequality ensures that we do not follow the computation backwards and the third equality ensures that we do not follow computationally non-valid steps. The last identity is used to ease the transition to a continuous optimization problem later. It is not necessarily required that $\weight_w < \weight_v$ whenever $w$ succeeds $v$ in the computation, although for some constructions below we enforce this property to ensure that the gradient descent method has strictly decreasing loss.

Next, we extend theses weighs to a continuous loss function $\ell: S \subset \real^{|V|} \to \real$ that can be optimized with gradient descent. To this end, we associate each vertex $v \in V$ with the standard unit basis vector $e_v \in \real^V$ and define the domain $S$ as the simplex 
\[
  S := \operatorname{conv}\{e_v | \, v \in V\}
\]
spanned by the basis vectors. The values at the vertices and barycenters $e_{vw}$ of each two vertices $(v,w)$ correspond to the graph weights:
\begin{equation}
  \begin{aligned}
    \ell(e_v) & = \weight_v, & v & \in V \\
    \ell( e_{vw}) & = \weight_{vw} , & v \ne w; \, v, w & \in V.
  \end{aligned}
  \label{eq:loss-vertices}
\end{equation}
The latter condition requires the symmetry of the edge weights in \eqref{eq:graph-edge-weights}.

Each iterate $x_k \in S$ of the gradient descent method will be a vertex. In order to control the gradient direction, we need to fill in values for $\ell$ in their neighborhood. To this end, we assume that 
\begin{align}
  & \ell\text{ is affine on the corners }S_v, & v & \in V(S).
  \label{eq:loss-affine}
\end{align}
A loss function with all properties in \eqref{eq:graph-edge-weights}, \eqref{eq:loss-vertices} and \eqref{eq:loss-affine} can be easily constructed by linear finite elements on a subdivision of the simplex $S$. Nonetheless, in this article, we use the related explicit functions from Appendix \ref{appendix:lagrange-basis}. These have two advantages: First, we do not need all finite element basis functions and therefore can avoid the construction of subdivisions for high dimensional simplices. Second, the explicit formulas demonstrate that they can be implemented by single layer $ReLU$ neural networks. Specifically, from Appendix \ref{appendix:lagrange-basis}, we have the functions $\ell_v$ and $\ell_{vw}$, which are linear in all corners $S_v^\mu$ and have the interpolation properties
\begin{align*}
  \ell_v(e_v) & = 1, &
  \ell_v(e_w) & = 0, &
  \ell_v(e_{uw}) & = 0, &
  \\
  \ell_{vw}(e_r) & = 0, &
  \ell_{vw}(e_{vw}) & = 1, &
  \ell_{vw}(e_{st}) & = 0, &
\end{align*}
for all $v \ne w$, $u \ne \{v,w\}$, $r \in V$ and $(s,t) \not\in \{(v,w), (w,v)\}$. A loss function satisfying \eqref{eq:loss-vertices} and \eqref{eq:loss-affine} is then given by
\begin{equation}
  \ell(x) 
  = \sum_{v \in V(S)} \weight_v \ell_v(x)
  + \sum_{v \ne w \in V} \weight_{vw}\ell_{vw}(x).
  \label{eq:tm-loss-tape-internal}
\end{equation}

Let us next assign some numeric values to the weights $w_v$ and $w_{vw}$ satisfying the requirements \eqref{eq:graph-edge-weights}. By assumption, we know that for all relevant inputs in $\inputClass$, the Turing machine halts in finite time, say with no more than $K-1$ steps. For each state $v \in V$ we run the Turing machine for at most $K$ steps and record the number $k(v)$ it took until it reaches a halting state or set $k(v) = K$ if it does not reach a halting state. For arbitrary numbers $W_0 < W_1 < \cdots < W_K = B$ we then assign
\begin{align*}
  \weight_{v} & = W_{k(v)} \\
  \weight_{vw} & = \left\{ \begin{array}{ll}
    \frac{1}{2} [W_{k(v)} + W_{k(w)}] & (v,w) \in E\text{ or }(w,v) \in E \\
    B & \text{else}
    \end{array}\right. .
\end{align*}

\paragraph{Gradient Descent}

Since the loss function is only defined on the simplex $S$, we use the conditional gradient method (Frank-Wolfe algorithm) for optimization. This is a modification of the gradient descent method, which only allows update directions that are compatible with the convex constraint $x \in S$. It is defined by
\begin{equation}
  \begin{aligned}
    d_k & = \argmin_{y \in S} \left\{ \partial_{y - x_k} \ell(x_k) \right\} - x_k\\
    x_{k+1} & = x_k - \alpha_k d_k.
  \end{aligned}
  \label{eq:cond-gd}
\end{equation}
More commonly, the update direction is defined by $d_k = y_k-x_k$ with  $y_k = \argmin_{y \in S} \dualp{\nabla \ell(x_k), y}$. This is equivalent to the one given above because $\partial_{y - x_k} \ell(x_k) = \dualp{\nabla \ell(x_k), y} - \dualp{\nabla \ell(x_k), x_k}$ and the latter term is independent of $y$. In our case $\ell$ is piecewise linear and may have kinks. Nonetheless, one-sided directional derivatives are well-defined and sufficient for the method above.

The learning rate is either constant $\alpha_k = 1$, or determined by a line search. For the latter, we need the extra condition that 
\begin{equation}
  \begin{aligned}
    \weight_v & > \weight_{vw} > \weight_w, &
    \ell &\phantom{=} \text{is linear on $\operatorname{conv}\{e_v, e_{vw}\}$ and $\operatorname{conv}\{e_{vw}, e_w\}$}, & 
    (v,w) & \in E
  \end{aligned}
  \label{eq:graph-loss-extra}
\end{equation}
to ensure that $\weight_w$ is indeed the minimizer along the gradient direction given by the edge from $e_v$ to $e_w$. For fixed learning rate this is not required since the loss is not necessarily strictly decreasing (at least for exact computation, which we consider in this paper).

\paragraph{Convergence}

The loss function is defined on the simplex $S \subset \real^V$ of a $|V|$ dimensional vector space. Since $|V|$ is the number of all possible states of the Turing machine with tape length bounded by $\tau$, the dimension is very large. It must be at least $|Q| 2^{d\tau} \tau^d$ with $|Q|$ states of the finite control, $2^{d \tau}$ possible tape contents and $\tau^d$ head positions. Practically, this is of course unrealistic and in Section \ref{sec:tm-descent-with-tape-variables} we consider a modified construction that scales linearly in the tape length.

We now show that gradient descent applied to our loss function traces the steps of the Turing machine $TM$.

\begin{proposition}
  \label{prop:tm-descent-no-tape-variable}
  Let $x_k$ be defined by the gradient descent method \eqref{eq:cond-gd} with initial value $x_0 = e_{v_0}$, $v_0 \in \inputClass \subset V(S)$ contained in the vertices of $S$. Assume that the loss function $\ell$ satisfies the conditions \eqref{eq:graph-edge-weights}, \eqref{eq:loss-vertices} and \eqref{eq:loss-affine}. Let the learning rate be constant $\alpha_k =1$ or defined by line search if \eqref{eq:graph-loss-extra} holds in addition.

For consecutive states $v_0, \dots, v_K$ of the Turing machine execution with initial state $v_0$ and halting state $v_K$, we have
\begin{align*}
  x_k & = e_{v_k}, & \text{for }k \le K. \\
  \intertext{If \eqref{eq:graph-loss-extra} holds, we have}
  x_k & = v_{v_K}, & \text{for }k > K
\end{align*}
in addition and the loss is strictly decreasing, i.e. $\ell(x_k) = \weight_{v_k}$ with $\weight_{v_0} > \cdots > \weight_{v_K}$ for $k =1, \dots, K$.
\end{proposition}

\begin{proof}

By induction, assume that $x_k = e_v$ is a non-halting state. Since the loss $\ell$ is linear on the corner $S_v$, Lemma \ref{lemma:simplex-derivative} implies that the updated direction of the gradient method is given by
\begin{equation}
  d_k = \argmin_{y \in S} \left\{ \partial_{y - x_k} \ell(x_k) \right\} - x_k= \frac{1}{\mu} \left( \argmin_{y \in V(S_v)} \ell(y) - x_k \right).
  \label{ep:prop:tm-descent-no-tape-variable-1}
\end{equation}
Since $v$ is a non-halting state, the interpolation property \eqref{eq:loss-vertices} and the weight properties \eqref{eq:graph-edge-weights} of the loss function, yield $d_k = \frac{1}{\mu} (e_{vw} - e_v) = e_w - e_v$, where $w \in V(S)$ is the successor state of $v$ in the Turing machine execution, i.e. $(v,w) \in E$. Therefore, the next step of the gradient method is given by
\[
  x_{k+1} = e_v - \alpha_k (e_w - e_v).
\]
If the learning rate is $\alpha_k = 1$, this directly yields $x_{k+1} = e_{v_{k+1}}$. In case the learning rate $\alpha_k$ is given by a line search, the extra condition \eqref{eq:graph-loss-extra} ensures that the loss function is piecewise linear along the line segments from $e_v$ to $e_{vw}$ and $e_{vw}$ to $e_w$ with values $\ell(e_v) > \ell(e_{vw}) > \ell(e_w)$. Thus, the line search follows this direction as far as possible without leaving the simplex $S$, resulting in the final point $x_{k+1} = e_w$.

If $x_k = e_{v}$ is a halting state, by the extra condition \eqref{eq:graph-loss-extra}, the minimizer of \eqref{ep:prop:tm-descent-no-tape-variable-1} is $e_v$ itself. Therefore, we have $x_{k+1} = x_k$, irrespective of the choice of $\alpha_k$.

\end{proof}

\subsection{Tracing Turing Machines: With Extra Tape Variable}
\label{sec:tm-descent-with-tape-variables}

The construction in Section \ref{sec:tm-descent-no-tape-variables} has the disadvantage that all possible relevant computations of the Turing machine must be known beforehand and hard-coded into the loss function. This quickly leads to an excessively high dimensional domain of the loss function. In this section, we consider an alternative construction, where the states of the simplex $S$ correspond to the finite control of the Turing machine only and the tape is stored in a extra variables that scales linearly with the tape length. The drawback of this method is that we have to use a fixed step size and the gradient descent iterates are no longer strictly decreasing.

\paragraph{The Loss Function}

The loss function is similar to the construction in Section \ref{sec:tm-descent-with-tape-variables} with the tape content split of into a new variable. We define a computational graph with vertices $V = Q \times \Gamma^d$ composed of the state of the finite control and copies of the tapes at the head positions. The directional edges $E$ correspond to legitimate computational steps, i.e. $(v,w) \in E$ if $v = (q,t)$ and $w = (q', t')$, with successor state $q' = \delta_1(q,t)$ of the finite control and arbitrary tape symbol $t'$ read from the tape.

As before, we associate each vertex $v = (q,t) \in V$ with a standard unit basis vector $e_v = e_{q,t}$ in the $|Q|2^d$-dimensional vector space $\real^V = \real^{Q \times \Gamma^d}$ and define the simplex $S$ as the convex hull of these basis vectors. To account for the $d$ tapes of length $\tau$, we introduce two extra variables: $T \in \real^{\tau \times d}$ with one component for each tape entry and $H \in \{0,1\}^{\tau \times d}$ with exactly one non-zero entry per tape indicating the head positions.

Note that the total number of vertices $2^d |Q|$ plus tape and head dimensions $2 \tau d$ scale linearly in the number of control states and the tape size. This is significantly less than the construction in Section \ref{sec:tm-descent-no-tape-variables} with at least an exponential number $|Q| 2^{d\tau} \tau^d$ of vertices.

As in Section \ref{sec:tm-descent-with-tape-variables}, for vertices $v=(q,t)$ and halting states $F$, we assign weights
\[
  \omega_v = \left\{ \begin{array}{rl}
    - b^3 \scale & q \in F \\
    0 & \text{else},
  \end{array}\right.
\]
with a global scaling factor $\scale$ and some constant $b>0$ to be chosen later. The weights distinguish halting from non-halting states, but are no longer decreasing while following a computational path. This cannot be avoided because the Turing machine may pass a vertex $v$ several times during its execution, with different tape content and head position. Since the tape is no longer included in $S$, we cannot distinguish these states with different weights. 

For two vertices $v=(q,t)$ and  $w=(q',t')$, the edge weights are given by
\begin{equation}
  \omega_{vw} = \left\{ \begin{array}{rl}
    - \left(b^3 + \sum_{\substack{i=1\\t'_i = t_i}}^d b \right) \scale & q' = \delta_1(q,t) \\
    0 & \text{else}.
  \end{array}\right.,
  \label{eq:loss-weight-TM-tape-external}
\end{equation}
The first choice is negative and therefore favors correct computational steps. Its structure is chosen to balance some quadratic terms that we add to the loss function later, for read and write operations to the tape variables $T$ and $H$. These quadratic terms may have zero gradients, in which case the extra sum in the definition of the edge weights $\omega_{vw}$ ensures correct successor states. The powers of $b$ are used to balance several contributions to the gradient. For later reference, we choose it sufficiently large so that
\begin{align}
  b^3 & \ge \frac{1}{2} (b^3 + db), &
  \frac{1}{2} b^3 - db & \ge 2d b^2, &
  b^2 & > b.
  \label{eq:loss-weight-constants-TM-tape-external}
\end{align}

We extend these weights to a continuous loss function $\ell_S$ on $S$ with the properties
\begin{equation}
  \begin{aligned}
    \ell_S (e_v) & = \omega_v, & v & \in V \\
    \ell_S \left( e_{vw}^{1/4} \right) & = \omega_{vw} , & (v,w) & \in E \\
    \ell_S \left( e_{vw}^{3/4} \right) & = \frac{1}{2} \omega_{vw} , & (v,w) & \in E \\
    \ell_S \left( e_{vw}^{1/4} \right) & = 0 , & (v,w) & \not\in E\text{ and }(w,v) \not\in E \\
    \ell_S \left( e_{vw}^{3/4} \right) & = 0 , & (v,w) & \not\in E\text{ and }(w,v) \not\in E \\
    \ell_S\text{ is affine on the Corners }&S_v^{1/4}, & v & \in V
  \end{aligned}
  \label{eq:loss-vertices-no-tape}
\end{equation}
Since $e_{vw}^{1/4} = e_{wv}^{3/4}$, the loss $\ell_S$ is not well defined if both $(v,w)$ and $(w,v)$ are legitimate computations in $E$, i.e. the Turing machine goes back and forth between two states in $S$. Without loss of generality, we assume that this cannot happen, i.e. that
\begin{equation}
  \text{For all vertices $v,w \in V$ not both $(v,w)$ and $(w,v)$ are contained in $E$}.
  \label{eq:no-back-step}
\end{equation}
This can easily be achieved by the following modification of the Turing machine $TM$: We triple the states to $\bar{Q} := \{[q,r] | \, q \in Q, \, r \in \{0,1,2\}\}$ and extend the transition function to
\[
  \bar{\delta}([q,r],t) = ([\delta(q,t), r+1\mod 3], \delta_2(q,t), \delta_3(q,t)).
\]
Thus, the extended Turing machine performs the exact same computations as the original one, with the exception that in each step the new index $r$ of state $[q,r]$ cycles through the numbers $0,1,2$. Therefore, if the state $v=([q,r], t])$ is followed by $w=([q',r+1 \mod 3], t')$ the latter is followed by some $[q'',r+2 \mod 3] \ne [q, r]$ because $r + 2 \mod 3 \ne r$. Therefore, we never have the sequence of states $v \to w \to v$, which implies \eqref{eq:no-back-step}.

In order to extend the loss function to the full simplex, we use the two functions $\ell_v^{1/4}$ and $\bar{\ell}_{vw}$ defined in Lemma \ref{lemma:lagrange-corner} and \eqref{eq:lagrange-interior-unsymmetric} in Appendix \ref{appendix:lagrange-basis}, which can be easily constructed with single layer $ReLU$ units. They have the properties
\begin{equation*}
  \begin{aligned}
    \ell_v^{1/4}(e_v) & = 1, 
    \\
    \ell_v^{1/4}(e_u) & = 0, & & & v  \ne & u \in V
    \\
    \ell_v^{1/4}\left( e_{tu}^{1/4} \right) & = 0, &
    \ell_v^{1/4}\left( e_{tu}^{3/4} \right) & = 0 , 
    & t,u & \in V
    \\
    \ell_v^{1/4} \text{ is linear on the corners }& S_u^{1/4}, & & & u & \in V.
  \end{aligned}
\end{equation*}
and
\begin{equation*}
  \begin{aligned}
    \bar{\ell}_{vw}(e_u) & = 0, & & & u & \in V 
    \\
    \bar{\ell}_{vw}\left( e_{vw}^{1/4} \right) & = 1, &
    \bar{\ell}_{vw}\left( e_{vw}^{3/4} \right) & = \frac{1}{2} 
    \\
    \bar{\ell}_{vw}\left( e_{tu}^{1/4} \right) & = 0, &
    \bar{\ell}_{vw}\left( e_{tu}^{3/4} \right) & = 0, & 
    \{t,u\} & \ne \{v,w\} 
    \\
    \bar{\ell}_{vw} \text{ is linear on the corners }& S_u^{1/4}, & & & u & \in V.
  \end{aligned}
\end{equation*}
Along the edge $(v,w)$, the function $\bar{\ell}_{vw}$ is piecewise linear as shown in Figure \ref{fig:basis-profile}. Therefore, a loss function satisfying all properties in \eqref{eq:loss-vertices-no-tape} is given by
\begin{align*}
  \ell_S(x) 
  = \sum_{v \in V} \omega_v \ell_v^{1/4}(x)
  + \sum_{v,w \in V} \omega_{vw} \bar{\ell}_{vw}(x).
\end{align*}

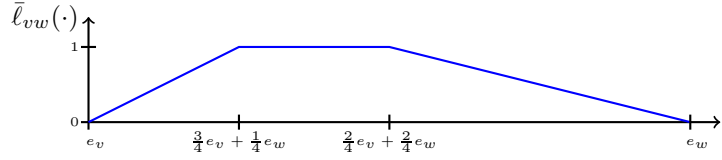
\begin{figure}
  \begin{center}
    \begin{tikzpicture}
      \newcommand{\h}{0.1}
    
      \draw[->,thick] (-\h,0) -- (8.4,0);
      \draw[->,thick] (0,-\h) -- (0,1.4) node[left] {$\bar{\ell}_{vw}(\cdot)$};
      \draw[blue, thick] (0,0) -- (2,1) -- (4,1) -- (8,0);

      \draw[thick] (2,\h) -- (2,-\h);
      \draw[thick] (4,\h) -- (4,-\h);
      \draw[thick] (8,\h) -- (8,-\h);
      \draw[thick] (-\h,1) -- (\h,1);
    
      \node[below] at (0,0) {\tiny $\phantom{\frac{1}{1}}e_v$};
      \node[below] at (2,0) {\tiny $\frac{3}{4} e_v + \frac{1}{4} e_w$};
      \node[below] at (4,0) {\tiny $\frac{2}{4} e_v + \frac{2}{4} e_w$};
      \node[below] at (8,0) {\tiny $\phantom{\frac{1}{1}}e_w$};

      \node[left] at (0,0) {\tiny $0$};
      \node[left] at (0,1) {\tiny $1$};
    \end{tikzpicture}
  \end{center}
  \caption{Function $\bar{\ell}_{vw}$ restricted to the line segment from $e_v$ to $e_v$.}
  \label{fig:basis-profile}
\end{figure}

The loss $\ell_S$ is used to trace the steps of the finite control via gradient descent, but in addition we need some extra terms for reading and writing to the tape variable $T$ and moving the head positions $H$. To this end, we first define some matrices $\tapeCurrent$, $\tapeNext$ to read out tape symbols from vertices $v = (q,t)$ and a bilinear map $\headShift$ to shift head positions and then add corresponding least squares terms to the loss function.

The matrices $\tapeCurrent, \tapeNext \in \real^{\tau \times d}$ yield the tape symbols of each vertex at the current and next state of the Turing machine, defined by
\begin{align*}
  \tapeCurrent e_{q,t} & = t \in \Gamma^d, &
  \tapeNext e_{q,t} & = \delta_{2}(q,t) \in \Gamma^d.
\end{align*}
The shift of the head positions $\headShift$ is defined by
\[
  [\headShift(e_{q,t})H]_{i,j} = h_{i+\delta_3(q,t), j},
\]
for basis vectors and then expanded to be bi-linear in both of its variables. It shifts all tape head indicators by the amount $\delta_3(q,t)$ determined by the current step of the Turing machine. Note that we always assume that the tapes are sufficiently large so that the head positions never reach its boundary. Hence, we may insert blank symbols if necessary.

We can now match the tape symbols $t$ in a vertex $x = e_v = e_{q,t}$ with tape symbols in the tape $T$. To this end, let us denote by subscripts $i$ columns of tape matrices, so that e.g. $T_i \in \real^\tau$ is the content of the $i$-th tape of the Turing machine. Then, adding a term $|T^T_i \cdot  [\headShift(x)H]_i - (\tapeCurrent x)_i|^2$ to the loss function matches the tape symbol of the $i$-th tape at the next head position with the tape symbol in the vertex $v$. We can do the same for all tapes by adding a term $\|\diag(T^T \headShift(x)H) - \tapeCurrent x\|^2$, where ``$\diag$'' is the vector of diagonal entries of the $d \times d$ matrix $T^T \headShift(x)H$.

However, with this expression we have no control if we read or write a symbol from the tape to the vertex. Therefore, we include extra ``stop gradient'' operations denoted by $\stopgrad{\cdot}$, which are readily available in current neural network libraries. Terms inside $\stopgrad{\cdot}$ are considered constant when differentiated, so that e.g. $\frac{d}{dx} f(x) \stopgrad{g(x)} = f'(x)g(x)$ as opposed to the correct $f'(x) g(x) + f(x) g'(x)$. With this extra operation, we can extend our above example to $\|\stopgrad{\diag(T^T \headShift(x)H)} - \tapeCurrent x\|^2$. Now the first term is considered constant so that minimizing it changes $x$ so that $\tapeCurrent x$ matches $\diag(T^T \headShift(x)H)$, i.e. we write the tape symbols at the head positions $\headShift(x)H$ into the component $t = \tapeCurrent x$ of $x$.

We add a global additive constant $c$ and several read/write operations to the loss function:
\begin{multline}
  \ell(x,T,H) = c + \ell_S(x)
  + \frac{1}{2} \scale \|\diag(T^T \stopgrad{H}) - \stopgrad{\tapeNext x}\|^2
  \\
  + \frac{1}{2} (4 b^2 \scale) \|\stopgrad{\diag(T^T \headShift(x) H)} - \tapeCurrent x\|^2
  + \frac{1}{2} \scale \|\stopgrad{\diag(\headShift(x) H)} - H\|^2.
  \label{eq:tm-loss-tape-external}
\end{multline}
The first least squares term writes $t' = \delta_2(q,t)$ to the current head position, the second term reads the content of the next head position into the next state $(q',t')$ and the last term moves the current head position to the next head position.

\paragraph{Gradient Descent}

The loss function is trained by the conditional gradient method
\begin{equation}
  \begin{aligned}
    d_k & = \argmin_{y \in S} \left\{ \partial_{y - x_k} \ell(x_k, T_k, H_k) \right\} - x_k\\
    x_{k+1} & = x_k + d_k \\
    T_{k+1} & = T_k - \frac{1}{\scale} \nabla_T \ell(x_k,T_k, H_k) \\
    H_{k+1} & = H_k - \frac{1}{\scale} \nabla_H \ell(x_k,T_k, H_k).
  \end{aligned}
  \label{eq:gradient-descent-tape-external}
\end{equation}
We slightly abuse notation and denote by $\nabla_T$ the gradient with respect to the writable tapes of $T$ only, excluding the read-only tapes. As in Section \ref{sec:tm-descent-no-tape-variables}, the loss function has kinks so that the gradient with respect to $x$ is not defined everywhere. Nonetheless, the one-sided directional derivatives used in the method are well defined at all points encountered during training.

\paragraph{Convergence}

The following proposition shows that minimizing the loss function traces the steps of the Turing machine.

\begin{proposition}
  \label{prop:tm-descent-tape-excluded}
  For $k=0, \dots, K$, let $v_k = (q_k, t_k)$ with tape variables $T_k$, $H_k$ be consecutive states of the Turing machine $TM$ with $t_k = \diag(T^T H)$, where $q_1, \dots, q_{K-1} \not\in F$ are non-halting states and $q_K \in F$ is a halting state. Assume that the loss function $\ell$ is defined by \eqref{eq:tm-loss-tape-external} and satisfies \eqref{eq:loss-weight-constants-TM-tape-external} and \eqref{eq:no-back-step}. Then $x_k := e_{v_k} \in \real^V$ and $T_k$, $H_k$ satisfy the gradient descent updates \eqref{eq:gradient-descent-tape-external} and in addition, we have
  \begin{equation}
    \begin{aligned}
      c \le \ell(x_k, T_k, H_k) & \le c + 8 b^2 d \scale + 3d\scale, & k & < K \\
      \ell(x_k, T_k, H_k) & \le c + 8 b^2 d \scale + 3d\scale - b^3\scale, & k & = K.
    \end{aligned}
    \label{eq:tm-descent-tape-excluded-stopping}
  \end{equation}
\end{proposition}

Note that unlike Proposition \ref{prop:tm-descent-no-tape-variable}, we cannot allow a line search for the gradient descent method any longer. Indeed non-halting states $x_k$ are vertices of the simplex $S$ and therefore the first component $\ell_S(x_k) = 0$ of the loss function is independent of $k$ and the remaining components are not necessarily decreasing.

The loss bounds in \eqref{eq:tm-descent-tape-excluded-stopping} are second order in $b$ for $k<K$ and third order for $k=K$. Therefore, for $b$ sufficiently large, the loss in a halting state is strictly smaller than the loss in any non-halting state, which can be used as a stopping criterion.

\begin{proof}

  Without loss of generality, we assume that the additive constant $c$ in the loss function is zero. By induction, assume that $x_k = e_v = e_{q,t}$, $T_k \in \Gamma^\tau = \{-1,1\}^\tau$, and $H_k \in \{0,1\}^\tau$ are iterates of the gradient descent method and that $q \not\in F$ is not an accepting state. The variable $x$ is updated by $x_{k+1} = x_k + d_k$ with direction
\[
  d_k + x_k = \argmin_{y \in S} \partial_{y-x_k} \left( \ell_S(x_k) + \frac{1}{2} (4 b^2 \scale)\|\tapeCurrent x_k - \stopgrad{\diag(T_k^T \headShift(x_k)H_k}\|^2 \right) , 
\]
where we have disregarded two least square terms in the loss function \eqref{eq:tm-loss-tape-external} because their $x$-gradient is zero by the stop gradient operation. Since all $\bar{\ell}_{vw}$ are linear in all corners $S_u^{1/4}$, $u \in V$, and we can disregard gradients of $\stopgrad{\diag(T_k^T \headNext x_k)}$, by Lemma \ref{lemma:simplex-derivative} with $\mu=1/4$, we have
\begin{equation}
  x_{k+1} = e_v + d_k = \argmin_{e_w \in V(S)} \ell_S\left( e_{vw}^{1/4} \right) + b^2 \scale \left( \tapeCurrent x_k - \diag(T_k^T \headShift(x_k)H_k) \right)^T \tapeCurrent e_w.
  \label{eq:proof:1:prop:tm-descent-tape-excluded}
\end{equation}
We have to show that the minimizer $e_w = e_{q',t'}$ is the next state of the Turing machine, i.e. $q' = \delta_1(q,t)$ and $t'$ contains a copy of the tape at the current head location shifted by $\delta_3(q,t)$. 

Let us first prove that $e_w$ has the correct state $q'$ by showing that the first term in \eqref{eq:proof:1:prop:tm-descent-tape-excluded} dominates the second. To this end, note that the numbers $(\tapeCurrent x_k)_i$, $(T_k)_i^T \cdot (\headShift(x_k)H_k)_i$ and $(\tapeCurrent e_w)_i$, for all tapes $i=1, \dots, d$ are tape symbols contained in $\{-1,1\}$ so that $b^2 \scale \left( (\tapeCurrent x_k)_i - (T_k)_i^T \cdot (\headShift(x_k)H_k \right)_i ) (\tapeCurrent e_w)_i \in \{-2 b^2 \scale,0,2 b^2 \scale\}$. It follows that the second summand $b^2 \scale \left( \tapeCurrent x_k - \diag(T_k^T \headShift(x_k)H_k) \right) \tapeCurrent e_w$ in \eqref{eq:proof:1:prop:tm-descent-tape-excluded} is contained in $\{2j b^2 \scale | \, j \in \{-d, \dots, d\} \}$. 

Since $q$ is a non-halting state, the vertex weight $\weight_v=0$ is zero and therefore for $\ell_S \left( e_{vw}^{1/4} \right)$ we have three possibilities: It is in the interval $\left[- (b^3 + db) \scale, - b^3 \scale \right]$ if $q',t'$ is a successor state of $q,t$ in the Turing machine execution, in $\left[ -\frac{1}{2} (b^3 + db) \scale, - \frac{1}{2} b^3 \scale \right]$ if $q',t'$ is a predecessor state of $q,t$ with the extra factor $1/2$ coming from the profile in Figure \ref{fig:basis-profile} and $0$ else, irrespective the tape symbols $t$ and $t'$. By assumption \eqref{eq:loss-weight-constants-TM-tape-external} we have $b^3 \ge \frac{1}{2} (b^3 + db)$ so that successor states are preferred over predecessor states. By the same assumption \eqref{eq:loss-weight-constants-TM-tape-external}, we also have $\frac{1}{2} (b^3 - db) \scale \ge 2d b^2 \scale$, i.e. the smallest possible gap between predecessor and successor states in the first gradient component $\ell(\cdot)$ is bigger than the maximal contribution $2db^2\scale$ of the second term in \eqref{eq:proof:1:prop:tm-descent-tape-excluded}. Hence, the minimizer $e_w$ of the directional derivative \eqref{eq:proof:1:prop:tm-descent-tape-excluded} with vertex $w = (q',t')$, is a successor state so that $q'=\delta_1(q,t)$. 

Next, we show that $t'$ contains the correct tape symbol. Since we already know that $q'$ is a successor state, we use the definition \eqref{eq:loss-weight-TM-tape-external} of the weights to simplify the gradient descent update \eqref{eq:proof:1:prop:tm-descent-tape-excluded} to $x_{k+1} = e_{q', t'}$ with 
\begin{equation*}
  t' = \argmin_{t' \in \real^d} -\left(b^3 + \sum_{\substack{i=1\\t'_i = t_i}}^d b \right) \scale 
  + b^2 \scale \left( \tapeCurrent x_k - \diag(T_k^T \headShift(x_k)H_k) \right)^T \tapeCurrent e_{q',t'}.
\end{equation*}
Eliminating constant summands and sorting the remaining terms with respect to components $i \in \{1, \dots, d\}$, we obtain
\begin{equation}
  t'_i = \argmin_{t'_i \in \real} - b \scale \delta_{t_i, t'_i} 
  + b^2 \scale \left( (\tapeCurrent x_k)_i - (T_k)_i^T \cdot (\headShift(x_k)H_k)_i ) \right) (\tapeCurrent e_{q',t'})_i,
  \label{eq:proof:2:prop:tm-descent-tape-excluded}
\end{equation}
where $\delta_{t_i,t'_i}$ is one if $t_i = t'_i$ and zero else. Since we want to read the tape symbol of tape $i$ at the next head position into the component $t'_i$ of the successor state $e_w$, we have to show that the minimizer of the last equation satisfies $t'_i = (T_k)_i^T \cdot (\headShift(x_k)H_k)_i$. We distinguish two cases:

\begin{enumerate}
  \item \emph{Case 1:} Let us assume that $t_i = (T_k)_i^T \cdot (\headShift(x_k)H_k)_i$, i.e. the tape symbol of the current vertex $x_k$ already matches the one of the successor state of the TM. In this case, the second term in \eqref{eq:proof:2:prop:tm-descent-tape-excluded} is zero and the term $\delta_{t_i,t'_i}$ in \eqref{eq:proof:2:prop:tm-descent-tape-excluded} ensures that $t_i'=t_i$, as required.

  \item \emph{Case 2:} Let us now assume that $t_i \ne (T_k)_i^T \cdot (\headShift(x_k)H_k)_i$. One easily verifies that the second term of \eqref{eq:proof:2:prop:tm-descent-tape-excluded} is $-2 b^2 \scale$ for $t'_i = (T_k)_i^T \cdot (\headShift(x_k)H_k)_i$ and $2 b^2 \scale$ for $t_i' \ne (T_k)_i^T \cdot (\headShift(x_k)H_k)_i$. The first term of \eqref{eq:proof:2:prop:tm-descent-tape-excluded} is either $b \scale$ or zero, depending on the choice of $t'_i$. Since $b^2 \scale > b\scale$ by assumption \eqref{eq:loss-weight-constants-TM-tape-external}, the second term dominates the choice and we obtain $t_i' = (T_k)_i^T \cdot (\headShift(x_k)H_k)_i$ as required.
\end{enumerate}

It remains to consider the gradient descent updates $T_{k+1}$ and $H_{k+1}$. By the stop gradient operations, we have
\begin{align*}
  \nabla_T \ell(x_k, T_k, H_k) 
  & = \nabla_T \frac{\scale}{2} \|\diag(T_k^T \stopgrad{H_k}) - \stopgrad{\tapeNext x_k}\|^2\\
  & = \nabla_T \frac{\scale}{2} \sum_{i \not\in \readOnly} |(T_k)_i^T \cdot \stopgrad{(H_k)_i}) - \stopgrad{(\tapeNext x_k)_i}|^2\\
  & = \scale \left\{(H_k)_i [(T_k)_i^T \cdot (H_k)_i - (\tapeNext x_k)_i] \right\}_{i \in W}
\end{align*}
Recall our convention that $\nabla_T$ is the gradient excludes the read-only tapes $\readOnly$. Therefore, we confine the sums and components above to the writable tapes not in $\readOnly$. The remaining read-only tapes are unchanged by both the Turing machine execution and the gradient descent updates.  Using that $(H_k)_i \in \{0,1\}^\tau$ is the indicator vector for the head location and $(\tapeNext x_k)_i = \delta_2(q,t)_i$, the last equation implies that 
\[
  (T_{k+1})_i = (T_k)_i - (H_k)_i [ (T_k)_i^T \cdot (H_k)_i - \delta_2(q,t)_i].
\]
so that $T_{k+1}$ is indeed the content of the tape in the next step of the Turing machine.

The update of the head location $H_k$ works analogously. We have
\begin{align*}
  \nabla_H \ell(x_k, T_k, H_k)
  & = \nabla_H \frac{\scale}{2} \|\stopgrad{\headShift(x_k) H_k} - H_k\|^2 \\
  & = - \scale [\headShift(x_k) H_k - H_k]
\end{align*}
and therefore
\begin{equation*}
  H_{k+1} 
  = H_k - \frac{1}{\scale} \nabla_H \ell(x_k, T_k, H_k)
  = H_k + [ \headShift(x_k) H_k - H_k ]
  = \headShift(x_k) H_k,
\end{equation*}
i.e. the new head position is shifted by $\headShift(x_k)$ as required.

In order to show the bounds \eqref{eq:tm-descent-tape-excluded-stopping}, note that all tape symbols are contained in $\{-1,1\}$, and the head locations are indicator vectors with entries in $\{0,1\}$. Hence, the last three summands in the loss \eqref{eq:tm-loss-tape-external} are in the intervals $[0, 2 d \scale]$ and $[0,8 b^2 d \scale]$, $[0, d\scale]$, respectively. The bound \eqref{eq:tm-descent-tape-excluded-stopping} then follows from the observation that on vertices $\ell_S$ is $-b^3 \scale$ in a halting state and zero else.

\end{proof}

\section{Application to Supervised Learning}
\label{sec:supervised-learning}

In this section, we prove the main result of the paper as stated in Section \ref{sec:main}, i.e. we construct an extended network that adjusts its weights according to the computation of a Turing machine during supervised learning. This is achieved by connecting the output of the Turing machine tracing from the last section with the weights of the primary network and to ensure that proper gradients are passed back into the tracing component by an optimization of a standard least squares loss.

\subsection{Construction of the Extended Network}
\label{sec:extended-network-construction}

The construction of the extended network $F(x, (s,t))$ is shown and briefly explained in Figure \ref{fig:nn-extended} in Section \ref{sec:extended-network-overview}. In this section, we fill in some details that we skipped over in its initial description.

The extended network mostly consists of two main branches connected by some extra ``switches''. The first branch contains the primary network $f_\theta(x)$ depending on the input $x$ and its parameters $\theta$, which are no longer trainable but outputs of some hidden layers. The second branch consists of a network that emulates the Turing machine $TM$ and is constructed similar to the Turing machine tracing in Propositions \ref{prop:tm-descent-no-tape-variable} and \ref{prop:tm-descent-tape-excluded}. The switches then pass one of these branches to the network output. The components of the extended network in Figure \ref{fig:nn-extended} are defined as follows:

\begin{itemize}
  \item \emph{$x$}: Neural network input.
  \item \emph{$z \in \real^{n \times n}$} holds a copy of the labels $y$ after the first training step. Note that we cannot pass $y$ into the Turing machine directly because it is only implicitly accessible through the loss function.
  \item \emph{$s,T,H$}: Trainable parameters, representing the Turing machine's state as in Sections \ref{sec:tm-descent-no-tape-variables}, \ref{sec:tm-descent-with-tape-variables}. They Contain $s \in S$ for the state state of the Turing machine and, depending on the construction, $T,H \in \real^{d\tau}$ for the tape symbols and head positions. 

  Although $s,T,H$ are mostly trainable weights, the inputs $x$ and $z$ are written on the read-only tapes of the Turing machine. All edges leading into and out of this node use quantization and de-quantization, except for the one to $f_{TM}$. No derivatives of the (de)quantization are required during training.
  \item \emph{$\theta$}: Weights of the primary network, read from the output tape of the Turing machine and not trainable.
  \item \emph{$f_\theta$}: The primary neural network, depending on input $x$ and weights $\theta$.
  \item \emph{$f_{TM}$}: Component of the network associated with the Turing machine $TM$. With initial $x$ and $y$ on its read-only tape, it halts with weights $\theta$ on the output tape and returns a vector $u$ that is fed into the downstream node $s_{net}$ as described below.
  \item \emph{$s_{net}$}: Switch between the Turing machine output $f_{TM}$ and the output of the primary network $f_\theta$.
  \item \emph{$s_{init}$}: Switch for reading the labels $y$ into $z$ and the remaining training.
  \item \emph{$out$}: Output of the extended network.
  \item Red boxes: Trainable weights.
  \item \emph{Dotted Arcs} denote ``stop-gradient'' operations, i.e. for a dotted arc from node $n_1$ to node $n_2$, we artificially set the gradient $\frac{\partial n_2}{\partial n_1} := 0$, even if that does not correspond the actual gradient. This means that we consider $n_1$ non-trainable or constant whenever it passes through $n_2$. This is a common feature implemented in most current deep learning libraries.
  \item Function arguments are ordered left to right with respect to incoming arcs in Figure \ref{fig:nn-extended} so that e.g. $s_{init}$ takes arguments in the order $s_{init}(s_{net}, z)$.
\end{itemize}

The Turing Machine $TM$ contains one read-only tape for the inputs $x$ and $y$. Unlike the other tapes, the read-only tape is not a trainable network weight but a hidden layer of width $\tau$ containing a quantization of the input $x$ and the weight $z$ (which will hold a copy of the labels $y$). Likewise, the weights $\theta$ of the primary network become a hidden layer containing a de-quantization of the output tape of the Turing machine. On a practical computer all variables are already stored in quantized format and can directly be written to and from the tapes. Formally, (de)quantizations can also be computed by a neural network from floating point numbers, see Appendix \ref{sec:quantization} for more details.

In the description of the main result, in particular for the loss \eqref{eq:lsq-loss}, we have split up the trainable weights as $(s,t)$ into variables $s$ that are restricted to a simplex and variables $t$ that are unrestricted. Now, with the full description of the extended network, we have $s=s$ and $\thz = T,H,z$. Likewise, the vector valued learning rate $\alpha$ in \eqref{eq:intro:gd-tm-supervised-learning} is split into three components $\alpha_T$, $\alpha_H$ and $\alpha_z$ for the three components $H,T,z$ of the trainable weights, respectively. 

After training, both switches $s_{init}$ and $s_{net}$ are ``open'', the network computes $f_\theta(x)$ with $\theta$ read from the Turing machine's tape.   Nonetheless, in the following sections, we describe the training process bottom-up in the network starting from the initial state.

\subsection{Reading the Labels}

In the first gradient descent step, we read the labels into the trainable variable $z$ and then shut off this reading process for the rest of the training. This is done with the switch $s_{init}(u,z)$, which, if fully turned on or off, lets either the first or second input pass through unchanged. The switch is triggered by the size of its second input variable $z$. Specifically, for two matrices $u,z \in \real^{n \times m}$ with the shape of the data $y$ and $\epsilon$ given by \eqref{eq:labels-size}, the switch has the following properties:
\begin{align*}
  s_{init}(u,0) & = 0, & \nabla_u s_{init}(u,0) & = 0, & \nabla_z s_{init}(u,0) & = I
  \\
  s_{init}(u,z) & = u, & \nabla_u s_{init}(u,z) & = I, & \nabla_z s_{init}(u,z) & = 0, &
  \text{for all }\|z\|_2^2 & \ge \epsilon,
\end{align*}
where $I \in \real^{(n \times m) \times (n \times m)}$ is the identity matrix. This can be easily realized by a cutoff function $\psi: [0, \infty) \to [0,1]$, with $\psi(t) = 1$ for $0 \le t < \frac{1}{3} \epsilon$ and $\psi(t) = 0$ for $\frac{2}{3} \epsilon < t$ and
\[
  s_{init}(u,z) := [1-\psi(\|z\|_2^2)] u + \psi(\|z\|_2^2) z.
\]
The cutoff function does not need to be differentiable in the interval $[\epsilon/3, 2\epsilon/3]$ and can be constructed e.g. from two $ReLU$ units. Let us now consider the first gradient descent step. With initial value $z_0 = 0$, the chain rule implies
\begin{align*}
  \nabla_z out & = \nabla_{s_{net}} s_{init}(s_{net},0) \nabla_z s_{net} + \nabla_z s_{init}(s_{net},0) I = I
  \\
  \nabla_\Box out & = \nabla_{s_{net}} s_{init}(s_{net},0) \nabla_\Box s_{net} + \nabla_z s_{init}(s_{net},0) 0 = 0
\end{align*}
for $\Box \in {s,T,H}$. Together with $out = 0$ for input $z_0 = 0$, it follows that
\begin{align*}
  \nabla_\Box \frac{1}{2} \|out-y\|_2^2 & = 0, &
  \nabla_z \frac{1}{2} \|out-y\|_2^2 & = out-y = -y &
\end{align*}
and therefore, the after one gradient descent step with learning rate $\alpha_z := 1$ for $z$, we have
\begin{align*}
  s_1 & = s_0, &
  T_1 & = T_0, &
  H_1 & = H_0, &
  z_1 & = z_0 - 1 (-y) = y.
\end{align*}
Since by assumption $\|y\|_2^2 > \epsilon$, for the next gradient descent steps we have 
\begin{align*}
  \frac{1}{2} \nabla_z\|out - y\|_2^2  
  & = (out-y)^T [\nabla_{s_{net}} s_{init}(s_{net},y) \nabla_z s_{net} + \nabla_z s_{init}(s_{net},y) I] 
  \\
  & = (out-y)^T \nabla_z s_{net} = 0,
\end{align*}
where we have used that $s_{net}$ only depends of $z$ via a stop gradient operation so that the latter gradient is zero. This directly implies that
\begin{equation*}
  \begin{aligned}
    z_0 & = 0, &
    z_k & = y, & 
    k & \ge 1.
  \end{aligned}
\end{equation*}
Likewise, we have
\begin{equation}
  \begin{aligned}
    \frac{1}{2} \nabla_\Box \|out-y\|^2  
    & = (out-y)^T [\underbrace{\nabla_{s_{net}} s_{init}(s_{net},y)}_{=I} \nabla_\Box s_{net} + \nabla_\Box s_{init}(s_{net},y) \underbrace{\nabla_\Box z}_{=0}] 
    \\
    & = (out-y)^T \nabla_\Box s_{net},
  \end{aligned}
  \label{eq:grad-s-init}
\end{equation}
for $\Box \in \{s,T,H\}$ so that the parameters
\begin{equation*}
  \begin{aligned}
    s_{k+1} 
    & = s_k - \left[\argmin_{\sigma \in S} \partial_{\sigma - s_k} \frac{1}{2} \|out - y\|^2 - s_k \right]
    \\
    & = s_k - [\argmin_{\sigma \in S} (out-y)^T \partial_{\sigma - s_k} s_{net} - s_k] \\
    T_{k+1} & = T_k - \alpha_T \frac{1}{2} \nabla_T \|out-y\|^2 = T_k - \alpha_T (out-y)^T \nabla_T s_{net} \\
    H_{k+1} & = H_k - \alpha_H \frac{1}{2} \nabla_H \|out-y\|^2 = H_k - \alpha_H (out-y)^H \nabla_H s_{net}
  \end{aligned}
\end{equation*}
train as if we would train the partial network with output $s_{net}$ only. Technically, as in \eqref{eq:directional-derivative} the node $f_{TM}$ inside $s_{net}$ only has one-sided directional derivatives, but one easily verifies that this does not alter the chain rule.

\subsection{Switching between Neural Network and Turing Machine}

The second switch $s_{net}$ selects between the primary network $f_\theta(x)$ and the Turing machine $f_{TM}(s, t)$. Just as the switch $s_{init}$, it has two inputs $s_{net}(f, u)$, with $f = f_\theta$ and $u = f_{TM}$ to shorten the notation below, and the properties
\begin{align}
  s_{net}(f,u) & = f, & \nabla_f f_{net}(f,u) & = I, & \nabla_u s_{net}(f,u) & = 0
  \label{eq:def-s-net-1}
  \intertext{if $\|u\|_2^2 < 2 \underline{B}$ and}
  s_{net}(f,u) & = u, & \nabla_f s_{net}(f,u) & = 0, & \nabla_u s_{net}(f,u) & = I
  \label{eq:def-s-net-2}
\end{align}
if $\|u\|_2^2 > 2 \overline{B}$ for some $\underline{B} < \overline{B}$ to be specified later. Similar to $s_{init}$ this can be realized by a smooth cutoff function $\phi: [0, \infty) \to [0,1]$ with $\phi(t) = 0$ for $t \le 2 \underline{B} + \delta$ and $\phi(t) = 1$ for $t \ge 2 \overline{B} - \delta$ for some $\delta > 0$ with $2 \underline{B} + \delta < 2 \overline{B} - \delta$ and
\[
  s_{net}(f,u) = [1-\phi(\|u\|_2^2)] f + \phi(\|u\|_2^2) u.
\]
Let $k \in \{1, \dots, K\}$ be a gradient descent step after the initialization in the first step and before the iterate $K$ when training is terminated by the stopping criterion \eqref{eq:intro:gd-tm-supervised-stopping} and $f_{TM}^k$ the outputs of the network node $f_{TM}$ in the corresponding steps. Starting from its initial state $s_0 = s_1$ till the $(K-1)$st step, one step before the halting state, we will ensure below that 
\begin{equation}
  \begin{aligned}
    \|f_{TM}^k\|_2^2 & > 2 \overline{B}, &  k & = 1, \dots, K-1
    \\
    \|f_{TM}^k\|_2^2 & < 2 \underline{B}, &  k & = K.
  \end{aligned}
  \label{eq:tm-loss-bounds}
\end{equation}
Thus before the Turing machine has halted, we have
\[
  out_k = f_{TM}^k
\]
and 
\begin{equation}
  \nabla_\Box out \stackrel{\eqref{eq:grad-s-init}}{=} \nabla_\Box s_{net} = \nabla_{f} s_{net} \nabla_\Box f_\theta + \nabla_{f_{TM}} s_{net} \nabla_\Box f_{TM} = \nabla_\Box f_{TM} 
  \label{eq:grad-tm}
\end{equation}
for $\Box \in \{s,T,H\}$, whereas after halting, we have
\[
  out = f_\theta.
\]
Again, we have used that the chain rule for the one-sided directional derivatives in $f_{TM}$ works as usual. Thus, we obtain the updates
\begin{equation}
  \begin{aligned}
    s_{k+1} 
    & = s_k - \left[\argmin_{\sigma \in S} \partial_{\sigma - s_k} \frac{1}{2} \|out - y\|^2 - s_k \right]
    \\
    & = s_k - [\argmin_{\sigma \in S} (out-y)^T \partial_{\sigma - s_k} f_{TM} - s_k] \\
    T_{k+1} & = T_k - \alpha_T \frac{1}{2} \nabla_T \|out-y\|^2 = T_k - \alpha_T (out-y)^T \nabla_T f_{TM} \\
    H_{k+1} & = H_k - \alpha_H \frac{1}{2} \nabla_H \|out-y\|^2 = H_k - \alpha_H (out-y)^H \nabla_H f_{TM}
  \end{aligned}
  \label{eq:gradient-descent-updates-2}
\end{equation}
for $k = 1, \dots, K-1$ so that while the Turing machine has not halted, the network output corresponds to the output of $f_{TM}$ and all weight updates are equivalent to training $f_{TM}$ alone, independent of the primary network $f_\theta$. After the Turing machine is finished, the switch $s_{net}$ deactivates the Turing machine branch $f_{TM}$ an outputs the primary network $f_\theta(x)$ with the values of $\theta$ that the Turing machine has written on its output tape.

\subsection{Training the Turing Machine}
\label{sec:train-TM}

Finally, we have to construct the node $f_{TM}$ such that the bounds \eqref{eq:tm-loss-bounds} of its input hold and gradient descent training traces the Turing machine calculations so that the final state of the output tape contains the weights $\theta = TM(x,y)$. To this end, we apply either Proposition \ref{prop:tm-descent-no-tape-variable} or Proposition \ref{prop:tm-descent-tape-excluded}, which both rely on a carefully crafted loss $\ell_{TM}$ associated with the Turing machine $TM$, which is different from the least squares loss in our application. We reconcile these two losses as follows: After reading the labels $y$ into $z_k = y$, we can calculate a vector $f_\perp(z)$ that is orthogonal to $z$ and has unit length, see Appendix \ref{appendix:orthogonal}. We then define the output of the node $f_{TM}$ by
\[
  f_{TM} = \sqrt{2 \ell_{TM}} f_\perp(z_k).
\] 
By choosing appropriate constants, we ensure that $\ell_{TM}$ is always non-negative below. This directly yields
\begin{align*}
  (out - y)^T \nabla_\Box f_{TM}
  & = \left(f_\perp(z) \sqrt{2 \ell_{TM}} - y \right)^T f_\perp(z) \nabla_\Box \sqrt{2 \ell_{TM}}
  = 2 \sqrt{\ell_{TM}} \nabla_\Box \sqrt{\ell_{TM}} \\
  & = \nabla_\Box \left( \sqrt{\ell_{TM}} \right)^2 
  = \nabla_\Box \ell_{TM}
\end{align*}
for $\Box \in \{s,T,H\}$ so that with \eqref{eq:gradient-descent-updates-2} we obtain the gradient descent updates
\begin{equation}
  \begin{aligned}
    s_{k+1} 
    & = s_k - [\argmin_{\sigma \in S} \partial_{\sigma - s_k} \ell_{TM} - s_k] \\
    T_{k+1} & = T_k - \alpha_T \nabla_T \ell_{TM} \\
    H_{k+1} & = H_k - \alpha_H \nabla_H \ell_{TM}
  \end{aligned}
  \label{eq:trace-TM}
\end{equation}
for the gradient descent steps $1, \dots, K-1$. These updates are identical to the gradient descent methods in Proposition \ref{prop:tm-descent-no-tape-variable} and Proposition \ref{prop:tm-descent-tape-excluded} (with appropriate choices of the learning rates $\alpha_T$ and $\alpha_H$) so that for $k=1, \dots, K-1$, the variables $s_k$, $T_k$ and $H_k$ have the same values if we train the loss $\frac{1}{2} \|out - y\|^2$ or alternatively the loss $\ell_{TM}$ directly. Note that we also have
\begin{equation}
  \|f_{TM}\|^2 = 2 \ell_{TM}  
  \label{eq:f_TM-size}
\end{equation}
which we will need for triggering the switch $s_{net}$, later.

We choose the stopping bound $B_{stop}$ in \eqref{eq:intro:gd-tm-supervised-stopping} such that $\|f_{TM(x,y)}(x) - y)\| \le B_{stop}$ and for the loss $\ell_{TM}$, we use one of the following two options:

\begin{enumerate}

  \item With the loss function of Proposition \ref{prop:tm-descent-no-tape-variable} and extra condition \eqref{eq:graph-loss-extra}, we include the full state of the Turing machine in the simplex $S$ with $|Q|$ possible states of the finite control, $2^\tau$ possible tape contents and $\tau$ possible head positions. Since the tape is already included in the simplex $S$ we do not need the extra variables $T,H$ and set $\tau = 0$ so that $T,H \in \real^0$ are trivial. The states $s_k \in S$ during training will be vertices of $S$ only so that we can easily read the weights $\theta$ of the primary network $f_\theta$ from $S$ directly. In summary, we have
  \begin{align*}
    S & \subset \real^{|Q| 2^{d\tau} \tau^d}, &
    T,H & \in \real^0.
  \end{align*}
  In order to choose the constants $\underline{B}$, $\overline{B}$ from the definition \eqref{eq:def-s-net-1}, \eqref{eq:def-s-net-2} of $s_{net}$ and $\weight_v$ and $B$ from the loss $\ell_{TM}$ in \eqref{eq:graph-edge-weights}, let $\weight_0 = \min_{v \in V(S)} \weight_v$ and $\weight_1 \in \min \{\weight_v | \, v \in V(S), \, \weight_v \ne \weight_0\}$ be the smallest and second but smallest weights. Then, we choose all weights $\weight_v$ and $\underline{B}$, $\overline{B}$ so that
  \begin{equation*}
    \weight_0 < B_{stop} = \underline{B} < \overline{B} < \weight_1.
  \end{equation*}
  With this choice, by Proposition \ref{prop:tm-descent-no-tape-variable} and \eqref{eq:trace-TM} the minimization of $\frac{1}{2}\|out - y\|^2$ traces the steps of the Turing machine $TM$ until we reach a halting state. Then by \eqref{eq:f_TM-size}, we have $\|f_{TM}\|^2 = 2 \ell_{TM} < 2 \underline{B}$ and by definition, the switch $s_{net}$ is flipped so that it lets the primary network $f_\theta$ pass to the output with weights $\theta = TM(x,y)$ read from the final tape of $f_{TM}$. 

  \item With the loss function $\ell_{TM}$ of Proposition \ref{prop:tm-descent-tape-excluded}, the finite control and tape symbols at the head positions are contained in the states of the simplex $S$ and the full tape and head positions in the variables $T$, $H$ so that 
  \begin{align*}
    S & \subset \real^{2^d |Q|}, &
    T,H & \in \real^{\tau \times d},
  \end{align*}
  where without loss of generality we assume that the Turing machine has two tapes: One for input and on for output. In order to select the constants $\underline{B}$ and $\overline{B}$ from definition \eqref{eq:def-s-net-1}, \eqref{eq:def-s-net-2} and $b$, $c$ and $\scale$ of the loss function \eqref{eq:loss-weight-TM-tape-external} and \eqref{eq:tm-loss-tape-external}, we first choose $b$ sufficiently large so that the condition \eqref{eq:loss-weight-constants-TM-tape-external} holds and $8 b^2 d \scale + 3d\scale - b^3\scale < 0$, which by \eqref{eq:tm-descent-tape-excluded-stopping} ensures that the loss for the halting states is strictly smaller than the loss for the non-halting states. Next, we select the positive global scaling factor $\scale$, the global additive constant $c$ and the bounds $\underline{B}$, $\overline{B}$ from \eqref{eq:tm-loss-bounds} such that the loss $\ell_{TM}(x,T,H)$ is positive and 
  \[
    c + 8 b^2 d \scale + 3d\scale - b^3\scale
    < B_{stop}
    = \underline{B} 
    < \overline{B}
    \le c.
  \]
  By \eqref{eq:tm-descent-tape-excluded-stopping} the left hand side is an upper bound for $\ell_{TM}(x,T,H)$ at halting states and the right and side a lower bound for non-halting states. Note that the left hand side is smaller than the right hand side by our choice of $b$. As for the alternative construction, this ensures that once a halting state is reached, the switch $s_{net}$ passes the primary network $f_\theta$ to the output.

\end{enumerate}

Once the switch $s_{net}$ is flipped, the primary network $f_\theta$ is passed to the output with weights $\theta = TM(x,y)$ read from the final state of $f_{TM}$'s tape. Since we have chose $B_{stop}$ such that $\|f_{TM(x,y)}(x) - y)\| \le B_{stop}$, the stopping criterion of the gradient descent method applies and no further training steps are executed.

\subsection{Number of Layers}

Finally, let us count the number of layers in the extended network. The switches $s_{init}$ and $s_{net}$ can be implemented with $3$ layers each: E.g. for 
\[
  s_{init}(u,z) := [1-\psi(\|z\|_2^2)] u + \psi(\|z\|_2^2) z
\]
we have one layer for the norm $\|z\|_2^2$, one for the cutoff function $\psi$ and one for the outer products of the scalar weights with the vectors $u$ and $z$.

Next, $f_{TM}$ is defined by
\[
  f_{TM} = \sqrt{2 \ell_{TM}} f_\perp(z_k).
\] 
We need one layer for the product of the square root with $f_\perp$ and can implement the latter two terms on parallel layers. The function $f_\perp$ can be implemented with $5$ layers as in Appendix \ref{appendix:orthogonal}. The square root is one layer and $\ell_{TM}$ can be implemented in one layer for the outer sum in \eqref{eq:tm-loss-tape-internal} or \eqref{eq:tm-loss-tape-external} and two layers for all involved functions from Section \eqref{appendix:lagrange-basis}. The least squares terms in \eqref{eq:tm-loss-tape-external} can be computed in one extra layer that is parallel to the others.

In summary, the extend network requires at most 12 layers plus the layers form the primary network $f$ and eventually some layers for quantization and de-quantization of the Turing machine tapes variables $T$.

\appendix

\section{Lagrange Basis}
\label{appendix:lagrange-basis}

Let $e_v$ with indices $v$ in some index set $V$ be an orthonormal basis of $\real^V$ with dimension $m = |V|$ and $S$ be the \emph{standard simplex} spanned by the vertices $e_v$, $v \in V$. Define the sub-simplex $S_v^\mu$ by cutting out a corner of the standard simplex $S$ with vertices $e_v$ and $e_{vw}^\mu = (1-\mu) e_v + \mu e_w$ for $v \ne w \in V$. 

We need a Lagrange basis function $\ell_v^\mu$ or equivalently a finite element hat function associated with this corner. Although the construction is standard, the following lemma provides an explicit formula, which entails that it can be easily constructed from a $ReLU$ network.

\begin{lemma}
  \label{lemma:lagrange-corner}
  For $v \in V$ and $0 < \mu \le 1$, define 
  \begin{align*}
    \ell_v^\mu & := \max \{0, (n_v^\mu)^Tx - d_\mu\}, &
    n_v^\mu & := \frac{1}{\mu m} \sum_{w \ne v} e_v - e_w, &
    d_\mu & := \frac{(1-\mu)m-1}{\mu m}.
  \end{align*}
  Then $\ell_v^\mu$ is linear on $S_v^\mu$ and 
  \begin{align*}
    \ell_v^\mu(e_v) & = 1, &
    \ell_v^\mu(e_{vw}^\mu) & = 0, &
    \ell_v^\mu(x) & = 0, &
    x & \in S \setminus S_v^\mu.
  \end{align*}
  
\end{lemma}

Geometrically, the scaled normal vector $\frac{\mu m}{m-1} n_v^\mu = e_v - \frac{1}{m-1} \sum_{w \ne v} e_w$ points from the barycenter of the face opposite of $e_v$ to $e_v$.

\begin{proof}

    Since $\mu$ is fixed throughout the proof, we abbreviate $n_v = n_v^\mu$, $d = d_\mu$ and $\ell_v = \ell_v^\mu$. For any $w \ne v$, we have
\begin{align*}
  e_v^T n_v^T = \frac{1}{\mu m} \sum_{u \ne v} e_v^T e_v - e_v^T e_u = \frac{m-1}{\mu m}, &
  \\
  e_w^T n_v^T = \frac{1}{\mu m} \sum_{u \ne v} e_w^T e_v - e_w^T e_u = - \frac{1}{\mu m}.
\end{align*}
This directly yields
\begin{align*}
  n_v^T e_v - d & = 1, &
  n_v^T e_{vw}^\mu - d & = 0, &
  n_v^T e_w - d & = - \frac{1-\mu}{\mu} \le 0
\end{align*}
and therefore
\begin{align*}
  \ell_v(e_v)& = 1, &
  \ell_v(e_{vw}^\mu)& = 0, &
  \ell_v(e_w)& = 0.
\end{align*}
Note that $n_v^T x - d = 0$ is zero on the face spanned by the vertices $e_{vw}^\mu$, $v \ne w$ and therefore $\ell_v^\mu$ is linear on $S_v^\mu$ and its complement $S \setminus S_v^\mu$. In fact $\ell_v^\mu$ is identically zero on the latter because it is a convex set with vertices $e_w$ and $e_{vw}^\mu$ with $w \ne v$ and $\ell_v^\mu$ is zero on all of these.

\end{proof}

In the following Lemma, we only consider $\mu = 1/2$ and use the abbreviations $S_v = S_v^{1/2}$ and $e_{vw} = e_{vw}^{1/2} = \frac{1}{2}(e_v + e_w)$. Recall that $V(\cdot)$ denotes the vertices of a simplex.

\begin{lemma}
  For $v, w \in V$ with $v \ne w$, and barycenter $b = \frac{1}{m} \sum_{u \in V} e_u$ of $S$, define 
  \begin{align*}
    \ell_{vw} & := \max \{0, n_{vw} x - c\} - \ell_v^{1/2}(x) - \ell_w^{1/2}(x), &
    n_{vw} & := 4( e_{vw} - b), &
    c & := 1 - \frac{4}{m},
  \end{align*}
  with $\ell_v^{1/2}$ and $\ell_w^{1/2}$ defined in Lemma \ref{lemma:lagrange-corner}. Then $\ell_{vw}$ is linear on each corner $S_v$ and 
  \begin{align*}
    \ell_{vw}(e_{vw}) & = 1, &
    \ell_{vw}(u) & = 0, &
    u & \in [V(S_v) \cup V(S_w)] \setminus \{e_{vw}\}, &
    \\
    & & \ell_{vw}(x) & = 0, &
    x & \in S_u, \, u \not\in \{v,w\}.
  \end{align*}
  
\end{lemma}

Geometrically, $n_{vw}x - c = 0$ defines a plane that contains all $(2|V|-3)$ vertices of $S_v$ and $S_w$ except $e_v$, $e_w$ and $e_{vw}$, even though these are more vertices than needed to span a plane. In addition, the plane cuts the simplex $S$ into two pieces, one that contains $S_v$ and $S_w$ and another that contains all other corners $S_u$, $u \ne v, w$.

\begin{proof}

First note that for any two indices $s,t \in V$, we have
\begin{align*}
  e_s^T b = e_{st}^T b = \frac{1}{m}
\end{align*}
Furthermore, one directly verifies that
\begin{align*}
  e_{vw}^T e_{vw} & = \frac{1}{2}, &
  e_{vw}^T e_{vu} & = \frac{1}{4}, &
  e_{vw}^T e_{uw} & = \frac{1}{4}, &
  e_{vw}^T e_{tu} & = 0, &
  \\
  e_{vw}^T  e_{v} & = \frac{1}{2}, &
  e_{vw}^T  e_{w} & = \frac{1}{2}, &
  e_{vw}^T  e_{u} & = 0, &
\end{align*}
for $t,u \not\in \{v,w\}$. This yields
\begin{align*}
  n_{vw}^T e_{vw}^T - c & = 1, &
  n_{vw}^T e_{vu}^T - c & = 0, &
  n_{vw}^T e_{uw}^T - c & = 0, &
  n_{vw}^T e_{tu}^T - c & = -1, &
  \\
  n_{vw}^T e_{v}^T - c & = 1, &
  n_{vw}^T e_{w}^T - c & = 1, &
  n_{vw}^T e_{u}^T - c & = -1, &
\end{align*}
and therefore, using Lemma \ref{lemma:lagrange-corner}, we conclude that
\begin{align*}
  \ell_{vw}(e_{vw})& = 1, &
  \ell_{vw}(e_{vu})& = 0, &
  \ell_{vw}(e_{uw})& = 0, &
  \ell_{vw}(e_{tu})& = 0, &
  \\
  \ell_{vw}(e_v)& = 0, &
  \ell_{vw}(e_w)& = 0, &
  \ell_{vw}(e_u)& = 0.
\end{align*}
for all $u \not\in \{v,w\}$. From the value of $n_{vw}x -c$ on the vertices of the simplices $S_t$, $t \in V$ and their convexity, we conclude that the kinks from the $ReLU$ units are not contained in the interior of any of the corners $S_t$, $t \in V$ and hence $\ell_{vw}$ is linear in each of them. This directly implies the lemma.

\end{proof}

\begin{lemma}
  For $v,w \in V$, $v \ne w$ and $0 < \mu \le 1$, let $S_{vw}^\mu$ be the ``unsymmetric corner'' or simplex with vertices $e_v$ and $(1-\mu) e_v + \mu e_w$ and $\frac{1}{2}(e_v + e_u)$ for $u \ne v,w$. Then there is a $n \in \real^V$ and $d \in \real$ such that the function $\ell^s_{vw,\mu}(x) = \operatorname{ReLU}(n^T x - d)$ satisfies
  \begin{align*}
    \ell_{vw}^\mu&\text{ is linear on } S_{vw}^\mu,&
    \ell_{vw}^\mu(e_v) & = 1
    \\
    \ell_{vw}^\mu&\text{ is zero on } S \setminus S_{vw}^\mu, &
    \ell_{vw}^\mu \big((1-\mu) e_v + \mu e_w\big) & = 0
    \\
    & & \ell_{vw}^\mu \left(\frac{1}{2}(e_v + e_u) \right) & = 0
  \end{align*}
  for all $u \ne v,w$.
  
\end{lemma}

\begin{proof}

One easily verifies that the vertices $(1-\mu) e_v + \mu e_w$ and $\frac{1}{2}(e_v + e_u)$ for $u \ne v,w$ of $S_{vw}^\mu$, or equivalently $V(S_{vw}^\mu) \setminus \{e_v\}$, are in general position so that they define a unique hyperplane given by $\{x \in \real^N : \, \bar{n}^T x + \bar{d} = 0\}$ with normal $\bar{n} \in \real^N$ and scalar $\bar{d}$.

Geometrically, the hyperplane cuts the simplex $S_{vw}^\mu$ from the larger simplex $S$. Therefore its corner vertex $e_v$ is not contained in the plane, which can easily be checked algebraically. Thus, we have $\bar{n}^T e_v + \bar{d} \ne 0$ and upon rescaling we obtain a new vector $n$ and scalar $d$ so that $n^Tx + d = 0$ defines the same hyperplane and we have $n^T e_v + d = 1$.

By construction, a point $x$ is contained in the convex set $S \setminus S_{vw}^\mu$ if and only if $n^T x + d < 0$. Together with the construction above, this shows all statements of the lemma.

\end{proof}

We need one more function $\bar{\ell}_{vw}$ with the properties
\begin{equation*}
  \begin{aligned}
    \bar{\ell}_{vw}(e_u) & = 0, & & & u & \in V 
    \\
    \bar{\ell}_{vw}\left( e_{vw}^{1/4} \right) & = 1, &
    \bar{\ell}_{vw}\left( e_{vw}^{3/4} \right) & = \frac{1}{2} 
    \\
    \bar{\ell}_{vw}\left( e_{tu}^{1/4} \right) & = 0, &
    \bar{\ell}_{vw}\left( e_{tu}^{3/4} \right) & = 0, & 
    \{t,u\} & \ne \{v,w\} 
    \\
    \bar{\ell}_{vw} \text{ is linear on the corners }& S_u^{1/4}, & & & u & \in V.
  \end{aligned}
\end{equation*}
It can be easily constructed from the functions in this section by 
\begin{equation}
  \bar{\ell}_{vw} = \ell_{vw} + \ell_v^{1/2} - \ell_{vw}^{1/4}.
  \label{eq:lagrange-interior-unsymmetric}
\end{equation}
Since the corner $S_w^{1/4}$ is contained in the unsymmetric corner $S_{wv}^{1/4}$ the function $\bar{\ell}_{vw}$ is linear on the corner $S_w^{1/4}$, as well as all other corners $S_u^{1/4}$, $u \in V$. In addition, all three functions $\ell_{vw}$, $\ell_v^{1/2}$ and $\ell_{vw}^{1/4}$ are lines on the line segments from $e_v$ to the barycenter $\frac{1}{2}(e_v + e_u)$ and from the barycenter to $e_u$ for all $u \ne v,w$. On these three points one directly verifies that $\bar{\ell}_{vw}$ is zero and therefore, it is zero on the entire line segment from $e_v$ to $e_u$, $u \ne w$. For all remaining $\bar{\ell}_{vw}(e_{tu}^{1/4})$ and $\bar{\ell}_{vw}(e_{tu}^{3/4})$, one can easily verify that their value is zero, as well. Finally, on the edge connecting $e_v$ to $e_w$, the function $\bar{\ell}_{vw}$ is piecewise linear as shown in Figure \ref{fig:basis-profile}.

\section{Quantization}
\label{sec:quantization}

In order to input data and labels into the Turing machine and later read the resulting weights $\theta$  from the output tape, we need to quantize and de-quantize floating point numbers to bit streams. From a practical perspective nothing needs to be done because on every computing device, data and weights are already represented by zeros and ones, which can be passed directly to and from the Turing machine.

Nonetheless, for the sake of completeness, we show by a rudimentary construction that this quantization/de-quantization can also be achieved purely by neural networks. To this end, given a number $x \in \real$, we compute a floating point approximation
\[
  x = (-1)^\sigma \left( \sum_{i=0}^m a_i 2^i \right) \cdot 2^{\sum_{j=0}^n b_j 2^j}
\]
with fixed sizes $m$, $n$ of the mantissa and exponent. Thus, we need two neural networks for  mapping $x$ to the binary $\sigma \times (a_0, \dots, a_m) \times (b_1, \dots, b_n)$ and its inversion. The latter can easily be implemented as a neural network with one exponential function as activation, or alternatively with a product of the terms $b_j2^j$. 

Since the operation $x \to \sigma, a,b$ is never differentiated for training the extended network in Figure \ref{fig:nn-extended}, we can use any standard algorithm for that purpose, as e.g. the following:

\begin{algorithmic}[1]
  \State $a=0$
  \State $\sigma = 0$ if $x>0$, else $\sigma = 1$.
  \For{$e \in 2^{\sum_{j=0}^n \hat{b}_j 2^j}$, $\hat{b} \in \{0,1\}^n$}
    \If{$2e > x \ge e$}
      \State $b = \hat{b}$
      \For{$i=m, \dots, 0$}
        \If{$x \ge 2^i e$}
          \State $a_i = 1$
	  \State $y = y - a_i 2^i e$
	\EndIf
      \EndFor
    \EndIf
    \State \textbf{return} $\sigma, a, b$
  \EndFor
\end{algorithmic}
With precomputed weights $2^i e$ with $i = 0, \dots, m$ and $e = 2^{\sum_{j=0}^n \hat{b}_j 2^j}$, $\hat{b} \in \{0,1\}^n$, we only need additions and if statements, which can be easily realized with Heaviside activation functions. The loops are finite and can be unrolled. If we want the network to be continuous, we can replace the Heaviside function with a piecewise linear approximation
\[
  \psi_\epsilon(x) = \left\{ \begin{array}{ll} 
    0 & x \le 0 \\
    \frac{x}{\epsilon} & 0 < x \le \epsilon \\
    1 & 1 \le x
  \end{array} \right.
\]
By choosing $\epsilon$ smaller than the machine accuracy, we can ensure that even with this approximation we obtain unchanged results for all floating point inputs $x$.

\section{Orthogonal Vector}
\label{appendix:orthogonal}

This appendix provides an algorithm for the function $f_\perp(z)$ used in Section \ref{sec:train-TM}. This function returns a vector $y$ that has unit length and is orthogonal to the input $z$. Of course it is easy to find an orthogonal vector, we just have to select one of them. To this end, we use two orthonormal vectors $v$ and $w$. Our first choice is
\begin{equation*}
  \begin{aligned}
    \bar{y} & = \argmin_{y} \frac{1}{2} \|y-v\|_2^2 & & \text{s.t.} & x^T y & = 0,
  \end{aligned}
\end{equation*}
given by $\bar{y} = v - \frac{x^Tv}{x^Tx} x$ and then upon normalization $y = \bar{y} / \|\bar{y}\|$. This procedure produces a non-zero $\bar{y}$ if and only if $v$ is not parallel to $x$. In the latter case (or better with a small safety margin if $x$ and $v$ are almost parallel) we repeat the same calculation with $w$, instead of $v$. 

This can be implemented with $5$ neural network layers: one switch between $v$ and $w$, one for the squared norm $x^Tx$, one for the division $x^Tv / x^Tx$ and likewise two more layers for the normalization of $\bar{y}$.

\bibliographystyle{abbrv}
\bibliography{train-tm}

\end{document}